\documentclass[11pt]{amsart}

\usepackage[T1]{fontenc}
\usepackage[utf8]{inputenc}
\usepackage[english]{babel}
\usepackage[margin=1in]{geometry}
\usepackage{indentfirst}
\usepackage{amsmath,amssymb,amsthm,mathtools}
\usepackage{graphicx,booktabs}
\usepackage[numbers,square,sort&compress]{natbib}
\let\originalcitet\citet
\renewcommand{\citet}[1]{%
  \def\currentcitekey{#1}%
  \def\boucheroncitekey{boucheron2013concentration}%
  \ifx\currentcitekey\boucheroncitekey
    Boucheron et al.~\citep{#1}%
  \else
    \originalcitet{#1}%
  \fi
}
\usepackage{xcolor}
\usepackage{url}
\definecolor{linkcol}{RGB}{0,70,160}
\definecolor{citecol}{RGB}{0,120,70}
\definecolor{urlcol}{RGB}{150,60,0}
\usepackage[
  colorlinks=true,
  linkcolor=linkcol,
  citecolor=citecol,
  urlcolor=urlcol,
  breaklinks=true,
  bookmarksopen=true,
  bookmarksnumbered=true,
  pdftitle={Beyond Multiplicity Correction: Sharp Retrieval Thresholds for Attention with Correlated Distractors}
]{hyperref}

\allowdisplaybreaks
\newtheorem{theorem}{Theorem}
\newtheorem{proposition}[theorem]{Proposition}
\newtheorem{corollary}[theorem]{Corollary}
\newtheorem{lemma}[theorem]{Lemma}
\theoremstyle{definition}

\newcommand{\E}{\mathbb E}
\newcommand{\Prob}{\mathbb P}
\newcommand{\R}{\mathbb R}
\newcommand{\pto}{\xrightarrow{\mathbb P}}

\newcommand{\Z}{\mathcal Z}
\newcommand{\frem}{f}
\newcommand{\pc}{P_{\mathrm{cl}}}
\newcommand{\pt}{P_{\mathrm{tok}}}
\newcommand{\ac}{A^{\mathrm{cnt}}_*}
\newcommand{\am}{A^{\mathrm{avg}}_*}
\newcommand{\hc}{h_{\mathrm{cnt}}}
\newcommand{\hm}{h_{\mathrm{avg}}}
\DeclareMathOperator{\Var}{Var}
\DeclareMathOperator{\Cov}{Cov}

\title[Sharp Retrieval Thresholds for Attention]{Beyond Multiplicity Correction: Sharp \\ Retrieval Thresholds for Attention with \\ Correlated Distractors}
\author{Heehyun Park}
\address{Department of Computer Science and Engineering, Pennsylvania State University}
\email{hbp5148@psu.edu}

\makeatletter
\renewcommand{\paragraph}{%
  \@startsection{paragraph}{4}%
  \z@{0.6\baselineskip}{-\fontdimen2\font}%
  {\normalfont\normalsize\bfseries}%
}
\makeatother
\begin{document}
\begin{abstract}
Redundant context affects attention in two distinct ways: multiplicity adds normalization mass, while score variation within correlated groups can produce extreme individual scores. We separate these effects in a planted retrieval model with exponentially many Gaussian distractors organized into supplied clusters whose number and size both grow exponentially with dimension. The model is a two-level generalized random energy model, whose classical free energy yields exact retrieval thresholds. Gaussian interpolation further shows that increasing positive correlation at fixed token count lowers the distractor free energy, unlike adding duplicates. 

Our main result establishes a sharp limitation of count correction, which subtracts the log cluster size from pre-softmax scores: this correction removes multiplicity but retains within-cluster extremes. Even when the inverse temperature is chosen after observing all scores, count correction requires a signal at least $\sqrt{2\alpha_1}$, where $\alpha_1$ is the exponential growth rate of the number of clusters. By contrast, averaging scores before exponentiation suppresses token-level fluctuations and lowers the optimal signal threshold to $\sqrt{2\rho\alpha_1}$, where $\rho$ is the within-cluster correlation. 

This gap persists after adding exponentially many singleton distractors, so group size alone does not identify the target, and extends, under explicit cumulant regularity assumptions, to fixed finite hierarchies of independent-increment scores. Finite-sample, stability, and participation results complement the asymptotic analysis. A controlled retrieval experiment with frozen sentence embeddings illustrates both sides of the distinction: averaging helps when the relevant item is a singleton, but can discard useful evidence when a relevant group contains multiple matching items.
\end{abstract}

\maketitle
\section{Introduction}
Repeated passages, related keys, and duplicated records can change attention even when the quality of an individual query--key score stays fixed. Two interventions are easily conflated: correlating a fixed collection of scores, and adding more copies of existing scores. Their effects need not have the same sign. A theory that attributes every redundancy failure to positive correlation risks recommending the wrong intervention.

We analyze one softmax retrieval step with a planted relevant item and a hierarchical Gaussian distractor ensemble. Our objective is to characterize the signal strength needed for its attention mass to approach one. We treat the hierarchy as a supplied partition, as could arise from document or record identifiers. Inferring semantic groups from learned keys is a separate problem. We also include exponentially many singleton distractors so that group sizes alone need not identify the target. The model permits exact comparisons of ordinary attention, subtraction of log cluster size, and averaging scores before exponentiation.

Our main result identifies what count correction leaves behind: dividing a cluster's weight by its size does not suppress its largest individual scores. Count correction therefore has a sharp signal threshold that cannot be lowered even by choosing the temperature after observing all scores (Theorem~\ref{thm:optimal}). Averaging the scores first achieves a strictly smaller threshold (Theorem~\ref{thm:averaging}). The hidden-singleton result removes identification from group size alone, and the deeper and non-Gaussian results extend this same comparison. Finite-sample and stability bounds describe its scope. The free energy and concentration transitions themselves come from classical random energy model (REM) and generalized random energy model (GREM) calculations \citep{derrida1980rem,derrida1985grem,bovier2004grem}; we provide self-contained proofs. Finally, a frozen-encoder experiment illustrates how the two operations differ on natural-sentence scores, including cases where averaging relevant scores hurts retrieval; it is not an empirical validation of the Gaussian asymptotic model.

\paragraph{Relation to prior work.} The Gibbs interpretation of attention and its connection to associative memory are established \citep{ramsauer2021hopfield}. REM methods already yield precise softmax and initialization analyses \citep{dohmatob2025softmax,giorlandino2025failure}. Long-context scaling depends on the score geometry \citep{chen2025critical}, and gap-counting unifies several concentration scales \citep{hayase2026unified}. Training-driven transitions \citep{lavie2026phase}, correlated Hopfield equilibria \citep{schimmenti2026bifurcation}, and clustered self-attention dynamics \citep{gao2026clustered} concern different observables or dynamics. Our observable is the mass of an externally planted item after a single retrieval step. This distinction does not make the underlying GREM calculation new.

Likewise, correcting sampling density is an established idea, including in PointConv \citep{wu2019pointconv}. ToMe explicitly accounts for multiplicity through a $+\log m$ adjustment after merging \citep{bolya2023tome}; Double-P uses size-weighted centroids to approximate token attention \citep{ni2026doublep}. We study a different objective: equal prior mass for supplied groups. LUCID addresses correlated-key interference through preconditioning \citep{duvvuri2026lucid}. We do not introduce a general density correction or token-merging architecture. Our focus is their sharp retrieval limits in an explicit probabilistic model.

\section{Model and normalization}
We model retrieval against a collection of distractor tokens that arrive in correlated clusters.  We begin with equally sized distractor clusters; Proposition~\ref{prop:types} later treats unequal cluster sizes through a fixed finite collection of distractor types.  The ambient dimension is $d$, and we study the high-dimensional limit $d\to\infty$ through integers.  A distractor is indexed by a pair $(a,j)$: the index $a$ identifies its cluster and $j$ identifies the token within that cluster.  There are $M_d$ clusters and $m_d$ tokens in every cluster, where
\begin{equation}
 M_d=\lceil e^{\alpha_1d}\rceil,\quad m_d=\lceil e^{\alpha_2d}\rceil,
 \qquad \alpha=\alpha_1+\alpha_2.
 \label{eq:population}
\end{equation}
Thus $\alpha_1$ is the exponential growth rate of the number of distinct clusters, $\alpha_2$ is the exponential growth rate of the multiplicity within each cluster, and the total number of distractors is $N_d=M_dm_d\asymp e^{\alpha d}$.

To represent similarity inside a cluster, assign to each cluster a shared Gaussian factor $G_a$ and to each token a token-specific Gaussian factor $H_{aj}$.  All of these variables are mutually independent $N(0,1)$ random variables, and the distractor score is
\begin{equation}
 X_{aj}=\sqrt{\rho d}\,G_a+\sqrt{(1-\rho)d}\,H_{aj},
 \label{eq:model}
\end{equation}
where $0<\rho<1$.  The term $G_a$ shifts every score in cluster $a$ together, whereas $H_{aj}$ distinguishes individual tokens within that cluster.  Hence each $X_{aj}$ is marginally $N(0,d)$, two different tokens in the same cluster have covariance $\rho d$ (and correlation $\rho$), and tokens in different clusters are independent.  The endpoints give useful intuition: $\rho=0$ would make all distractor scores independent, while $\rho=1$ would make the tokens within a cluster exact score duplicates.  Equation~\eqref{eq:model} is therefore a two-level Gaussian random-energy model for repeated or near-duplicate distractors.

Alongside these distractors, let $X_*$ denote the score of the single relevant, or planted, item.  At inverse temperature $\beta>0$, ordinary softmax attention assigns that item the mass
\begin{equation*}
 \Z_d(\beta)=\sum_{a=1}^{M_d}\sum_{j=1}^{m_d}e^{\beta X_{aj}},
 \quad L_d(\beta)=d^{-1}\log\Z_d(\beta),\qquad
 A_*=\frac{e^{\beta X_*}}{e^{\beta X_*}+\Z_d(\beta)}.
\end{equation*}
Here $\Z_d$ is the total unnormalized attention weight of the distractors and $L_d$ is its log weight per dimension. The exact normalized log-odds against the planted item are
\begin{equation}
 \frac{1}{\beta d}\log\frac{1-A_*}{A_*}
 =\frac{L_d(\beta)}{\beta}-\frac{X_*}{d}.
 \label{eq:logodds}
\end{equation}
Thus a negative limiting right-hand side gives retrieval, while a positive limit gives failure. We assume only $X_*/d\pto s\in\R$: the constant $s$ is the planted item's asymptotic score advantage on the order-$d$ scale. We do not require $X_*$ to be independent of the distractors. Retrieval means $A_*\pto1$, and all probability statements refer to the score ensemble.

The score model can be realized as a query--key dot product $X_{aj}=q^\top k_{aj}$. Here $q\in\R^d$ is a deterministic query vector with $\|q\|^2=d$, and $k_{aj}\in\R^d$ is the key of token $j$ in cluster $a$. Let $z_a\sim N(0,I_d)$ be the random vector shared by cluster $a$, and let $\epsilon_{aj}\sim N(0,I_d)$ be the token-specific random vector; all vectors $z_a$ and $\epsilon_{aj}$ are mutually independent. Setting $k_{aj}=\sqrt\rho\,z_a+\sqrt{1-\rho}\,\epsilon_{aj}$ gives the score distribution in \eqref{eq:model}. A planted key $k_*$ can be chosen so that $q^\top k_*=sd$. Under conventional scaled dot-product attention, the pre-softmax score is $t_d q^\top k/\sqrt d$ for a dimension-dependent sharpening factor $t_d$, so our $\beta$ equals $t_d/\sqrt d$. Thus fixed $\beta$ corresponds to $t_d=\beta\sqrt d$, and cannot be identified with fixed conventional scaling. A learned query depending on these keys need not satisfy the assumed law.

Define the one-level free energy for variance $vd$,
\begin{equation}
 \frem(\beta;a,v)=
 \begin{cases}a+\beta^2v/2,&\beta\le\sqrt{2a/v},\\
 \beta\sqrt{2av},&\beta\ge\sqrt{2a/v},\end{cases}
 \qquad \frem(\beta;a,0)=a.
 \label{eq:rem}
\end{equation}
We use inverse temperature: larger $\beta$ means sharper attention.

\section{Exact thresholds and the sign of correlation}

\begin{theorem}[Classical two-level free energy, specialized]\label{thm:free}
For fixed $\beta>0$, define
\begin{equation}
 F(\beta)=
 \begin{cases}
 \frem(\beta;\alpha,1),&\rho\le\alpha_1/\alpha,\\
 \frem(\beta;\alpha_1,\rho)+\frem(\beta;\alpha_2,1-\rho),
       &\rho>\alpha_1/\alpha.
 \end{cases}\label{eq:free}
\end{equation}
Then the normalized distractor log-partition function $L_d(\beta)=d^{-1}\log\Z_d(\beta)$ converges to $F(\beta)$ in probability and in $L^1$.
Equivalently, with $I_1(x)=x^2/(2\rho)$ and $I_2(y)=y^2/[2(1-\rho)]$,
\begin{equation}
 F(\beta)=\max_{\substack{I_1(x)\le\alpha_1\\I_1(x)+I_2(y)\le\alpha}}
 \{\alpha-I_1(x)-I_2(y)+\beta(x+y)\}.
 \label{eq:variational}
\end{equation}
Consequently, if $X_*/d\pto s\in\R$, then $A_*\pto1$ when $s>F(\beta)/\beta$ and $A_*\pto0$ when $s<F(\beta)/\beta$.
\end{theorem}

In \eqref{eq:variational}, $x$ is the order-$d$ deviation of the cluster-level parent shared by all tokens in a cluster, while $y$ is the token-specific child deviation. The first constraint requires that such a rare parent exist among the $M_d$ clusters; the second is the total entropy constraint requiring that a parent--child path with deviations $(x,y)$ exist among all $M_dm_d$ tokens. Replacing the second constraint with $I_2(y)\le\alpha_2$ for every parameter choice gives an incorrect answer in the merged regime.

To describe the freezing transitions, define the three inverse-temperature scales
\begin{equation*}
 b_1=\sqrt{2\alpha_1/\rho},\qquad
 b_2=\sqrt{2\alpha_2/(1-\rho)},\qquad b=\sqrt{2\alpha}.
\end{equation*}
If $\rho\le\alpha_1/\alpha$, the parent constraint imposes no additional restriction on the optimizer and the free energy matches that of an i.i.d. REM with the same marginal variance, although the scores remain correlated. If $\rho>\alpha_1/\alpha$, equivalently $b_1<b_2$, parent and child freezing occur separately, and the three branches of $F$ are
\begin{equation}
 F(\beta)=\begin{cases}
 \alpha+\beta^2/2,&\beta\le b_1,\\
 \beta\sqrt{2\rho\alpha_1}+\alpha_2+\beta^2(1-\rho)/2,
      &b_1\le\beta\le b_2,\\
 \beta\big(\sqrt{2\rho\alpha_1}+\sqrt{2(1-\rho)\alpha_2}\big),&\beta\ge b_2.
 \end{cases}\label{eq:split}
\end{equation}
The surface $\beta s=F(\beta)$ is deliberately excluded from the retrieval dichotomy. The first-order limit alone does not determine its behavior.

\paragraph{Finite-dimensional refinement.} Theorem~\ref{thm:free} gives the limiting threshold as $d\to\infty$. Proposition~\ref{prop:concentration} gives a separate finite-$d$ refinement: it bounds the deviation of $L_d$ from its mean and combines this with the deviation of $X_*/d$ from $s$ to give an explicit retrieval margin. This proposition is not needed to prove the later threshold results; it quantifies the finite-size scope of the asymptotic statement.

\begin{theorem}[Correlation comparison]\label{thm:comparison}
At any finite $d$, with fixed $M_d$ clusters and $m_d$ tokens per cluster, and for $0<\rho<1$, define the distractor-only token weights $p_{aj}=e^{\beta X_{aj}}/\Z_d$ and cluster weights $P_a=\sum_jp_{aj}$. Then
\begin{equation*}
 \frac{\partial}{\partial\rho}\E L_d(\beta)
 =-\frac{\beta^2}{2}\E\left[\sum_aP_a^2-\sum_{a,j}p_{aj}^2\right]\le0.
\end{equation*}
In particular $F(\beta;\rho)\le F(\beta;0)=\frem(\beta;\alpha,1)$. For a common limiting signal $s$, i.i.d. success $\beta s>\frem(\beta;\alpha,1)$ implies success for every $\rho\in[0,1]$.
\end{theorem}
The finite-$d$ derivative compares models with the same number of tokens and the same marginal $N(0,d)$ distribution for every score; only the correlation between tokens in the same cluster changes. It says that the expected log distractor weight decreases as $\rho$ increases. It does not say that $\Z_d$ decreases for every realized score array, and an ordering of the mean alone does not order finite-sample retrieval probabilities. The asymptotic threshold comparison additionally uses concentration and convergence to the deterministic free energy $F$. Adding a literal duplicate is a different operation: it adds another positive term to $\Z_d$ and therefore increases the softmax denominator.

\paragraph{What actually condenses?} We call
\begin{equation*}
 \pc:=\sum_aP_a^2,
 \qquad
 \pt:=\sum_{a,j}p_{aj}^2
\end{equation*}
the \emph{cluster participation} and \emph{token participation}, respectively. Throughout the paper, participation refers to these squared-mass sums. Equivalently, they are the conditional probabilities that two independent draws from the same distractor attention distribution choose the same cluster or the same token. If the attention mass is uniform over $K$ clusters or tokens, the corresponding participation equals $1/K$, so its reciprocal is the associated effective number of attended objects. Their expectations have the following limits (proof in Appendix~\ref{app:participation}):
\begin{equation}
 (\E\pc,\E\pt)\longrightarrow
 \begin{cases}
 ((1-b_1/\beta)_+,(1-b_2/\beta)_+),&b_1<b_2,\\
 ((1-b/\beta)_+,(1-b/\beta)_+),&b_1\ge b_2.
 \end{cases}\label{eq:participation}
\end{equation}
Here $x_+=\max(x,0)$. In the intermediate phase $b_1<\beta<b_2$, cluster participation is nonvanishing but maximum token mass tends to zero. Neither maximum cluster mass nor maximum token mass tends to one at any fixed finite $\beta$: their respective participation expectations remain strictly below one. These are distractor phases, distinct from the external planted-item transition.

\section{Count correction has a finite optimal temperature}
Ordinary softmax assigns one unnormalized weight $e^{\beta X_{aj}}$ to every token, so a cluster can receive more total mass simply because it contains more tokens. Count correction removes this multiplicity advantage using the supplied partition: a token with score $X$ in a group of size $m$ receives corrected weight $e^{\beta X}/m$, equivalently subtracting $\log m$ from its pre-softmax score $\beta X$.

We first consider the balanced model in which every distractor cluster has size $m_d$ and the planted item forms a singleton group. Each distractor weight is therefore divided by $m_d$, while the planted weight is unchanged. Define the count-corrected planted mass and its fixed-temperature signal threshold by
\begin{equation}
 \ac=\frac{e^{\beta X_*}}
 {e^{\beta X_*}+\Z_d(\beta)/m_d},\qquad
 \hc(\beta)=\frac{F(\beta)-\alpha_2}{\beta}.
 \label{eq:count}
\end{equation}
Applying the log-odds identity \eqref{eq:logodds} with the corrected distractor denominator $\Z_d/m_d$ gives
\begin{equation*}
 \frac{1}{\beta d}\log\frac{1-\ac}{\ac}\pto\hc(\beta)-s.
\end{equation*}
Thus count-corrected retrieval succeeds when $s>\hc(\beta)$ and fails when $s<\hc(\beta)$. By Theorem~\ref{thm:free}, the ordinary-attention threshold is $F(\beta)/\beta$; since $d^{-1}\log m_d\to\alpha_2$, count correction lowers it by exactly $\alpha_2/\beta$.

The correction does not change relative weights within or between the equally sized distractor clusters. If every group, including the relevant group, had the same size, the same additive shift would be applied to every pre-softmax score and the softmax weights would not change. In the present basic model the planted item is the unique singleton, so group size alone reveals its location; Corollary~\ref{cor:singletons} removes this shortcut by adding exponentially many singleton distractors.

\begin{theorem}[Optimal correction and adaptive-temperature lower bound]\label{thm:optimal}
For $\alpha_1,\alpha_2>0$ and $0<\rho<1$, the fixed-temperature retrieval threshold $\hc(\beta)$ is minimized uniquely at $\beta_{\mathrm{opt}}=\sqrt{2\alpha_1}$:
\begin{equation}
 \min_{\beta>0}\hc(\beta)=\sqrt{2\alpha_1}.
 \label{eq:optimal}
\end{equation}
Thus, no fixed $\beta$ retrieves when $s<\sqrt{2\alpha_1}$, whereas the fixed choice $\beta=\sqrt{2\alpha_1}$ retrieves when $s>\sqrt{2\alpha_1}$. Moreover, even if $\beta$ is selected after observing the realized distractor scores,
\begin{equation}
 \inf_{\beta>0}\frac{\log\Z_d(\beta)-\log m_d}{\beta d}
 \pto\sqrt{2\alpha_1}.\label{eq:adaptive}
\end{equation}
If $s<\sqrt{2\alpha_1}$, then $\sup_{\beta>0}\ac(\beta)\pto0$, even when $\beta$ is chosen after observing all scores. If $s>\sqrt{2\alpha_1}$, the fixed choice in \eqref{eq:optimal} gives $\ac\pto1$.
\end{theorem}
\begin{proof}[Proof idea; complete proof in Appendix~\ref{app:optimal}]
Before the first freezing point, $\hc(\beta)=\alpha_1/\beta+\beta/2$, whose derivative $-\alpha_1/\beta^2+1/2$ changes from negative to positive at $\beta=\sqrt{2\alpha_1}$. This value occurs before the first freezing point in both regimes of \eqref{eq:free}, and direct differentiation shows that every later formula for $\hc$ is increasing. For the data-dependent choice of $\beta$, convergence is uniform when $\beta$ stays in a fixed bounded interval. Jensen's inequality rules out $\beta\to0$, while the maximum distractor score rules out $\beta\to\infty$; hence a minimizing sequence cannot escape to either endpoint.
\end{proof}
Increasing $\beta$ makes the softmax place more weight on the largest scores. For count-corrected weights $e^{\beta X}/m$, the required signal $\hc(\beta)$ decreases until $\beta=\sqrt{2\alpha_1}$ and then increases; choosing a larger $\beta$ therefore requires a stronger planted signal rather than improving retrieval. With $M_d$ independent $N(0,d)$ distractor scores, the smallest $s$ permitting retrieval after optimizing $\beta$ is also $\sqrt{2\alpha_1}$, although the required $s$ at a particular fixed $\beta$ generally differs between the two models.

If one could instead use only the $M_d$ noise-free cluster components $\sqrt{\rho d}G_a$, the smallest required signal would be $\sqrt{2\rho\alpha_1}$. Count correction does not attain this smaller value because it divides out cluster size but retains the token-specific noise $\sqrt{(1-\rho)d}H_{aj}$; the averaging method below suppresses that noise.

\begin{proposition}[Upper bound on finite-sample failure probability]\label{prop:finite}
If $X_*=sd$ deterministically, then for any $0<\varepsilon<1$,
\begin{equation}
 \Prob\{\ac\le1-\varepsilon\}
 \le \min\left\{1,\frac{1-\varepsilon}{\varepsilon}
 M_d\exp\left[d(\beta^2/2-\beta s)\right]\right\}.
 \label{eq:finite}
\end{equation}
For $s>0$, the exponent is minimized at $\beta=s$, giving $M_d e^{-ds^2/2}(1-\varepsilon)/\varepsilon$. The bound needs only marginal $N(0,d)$ distractor scores and does not assume their independence.
\end{proposition}
Let $\alpha_{1,d}=d^{-1}\log M_d$. The upper bound in \eqref{eq:finite} decreases exponentially with $d$ when
\begin{equation}
 \alpha_{1,d}+\frac{\beta^2}{2}-\beta s<0.
 \label{eq:finitecondition}
\end{equation}
This is possible exactly when $s^2>2\alpha_{1,d}$, and then \eqref{eq:finitecondition} is equivalent to
\begin{equation*}
 s-\sqrt{s^2-2\alpha_{1,d}}<\beta<s+\sqrt{s^2-2\alpha_{1,d}}.
\end{equation*}
Thus this upper bound is useful only when $\beta$ is neither too small nor too large. A small $\beta$ leaves substantial total mass on the many distractors, while a large $\beta$ amplifies rare high distractor scores in the moment bound. The proposition gives no failure claim outside this interval; only this particular upper bound ceases to guarantee success there.

\section{Averaging removes token-specific score noise}
Count correction replaces each cluster's total weight $\sum_j e^{\beta X_{aj}}$ by the average of its exponentiated token scores, $m_d^{-1}\sum_j e^{\beta X_{aj}}$. This removes the factor caused by cluster size, but a single unusually large token-specific component $H_{aj}$ can still dominate the average. Averaging scores first performs a different operation: for each cluster $a$, compute
\begin{equation*}
 \bar X_a=\frac{1}{m_d}\sum_{j=1}^{m_d}X_{aj},
\end{equation*}
and then assign one attention weight $e^{\beta\bar X_a}$ to that cluster. The planted item remains a singleton with score $X_*$. Define
\begin{equation*}
 \Z_d^{\mathrm{avg}}=\sum_ae^{\beta\bar X_a},\quad
 \am=\frac{e^{\beta X_*}}{e^{\beta X_*}+\Z_d^{\mathrm{avg}}},
 \quad \hm(\beta)=\frac{\frem(\beta;\alpha_1,\rho)}{\beta}.
\end{equation*}
Here $\Z_d^{\mathrm{avg}}$ is the total distractor weight after replacing every cluster by its mean score, $\am$ is the resulting planted-item mass, and $\hm(\beta)$ will be the required planted signal at fixed $\beta$.

Under the Gaussian score model in \eqref{eq:model}, averaging gives
\begin{equation*}
 \bar X_a=\sqrt{\rho d}\,G_a+
 \sqrt{(1-\rho)d}\,\frac{1}{m_d}\sum_{j=1}^{m_d}H_{aj}.
\end{equation*}
The second term is Gaussian with variance $(1-\rho)d/m_d$, so it becomes negligible on the order-$d$ scale because $m_d$ grows exponentially. Thus the averaged cluster scores asymptotically behave like the $M_d$ independent shared components $\sqrt{\rho d}G_a$. This is the Gaussian model used for the asymptotic formulas in the theorem below; the finite-score comparison in its first sentence does not require Gaussian scores.
\begin{theorem}[Count correction versus averaging]\label{thm:averaging}
For every fixed collection of token scores, $\Z_d^{\mathrm{avg}}\le\Z_d/m_d$ and $\am\ge\ac$. Under the Gaussian score model in \eqref{eq:model},
\begin{equation*}
 d^{-1}\log\Z_d^{\mathrm{avg}}\pto\frem(\beta;\alpha_1,\rho),\qquad
 \frac{1}{\beta d}\log\frac{1-\am}{\am}\pto\hm(\beta)-s.
\end{equation*}
Thus averaging retrieves when $s>\hm(\beta)$ and fails when $s<\hm(\beta)$. Optimizing over $\beta$ gives
\begin{equation}
 \inf_{\beta>0}\hm(\beta)=\sqrt{2\rho\alpha_1}
 <\sqrt{2\alpha_1}=\inf_{\beta>0}\hc(\beta).
 \label{eq:gap}
\end{equation}
Consequently, if $\sqrt{2\rho\alpha_1}<s<\sqrt{2\alpha_1}$, averaging retrieves with any fixed $\beta\ge b_1$, while count correction fails even if its $\beta$ is selected after observing all distractor scores.
\end{theorem}
The first two inequalities follow from Jensen's inequality. They mean that, for the same fixed token scores, averaging before exponentiation gives a distractor denominator no larger than count correction and therefore gives the planted item at least as much attention mass. The complete proof is in Appendix~\ref{app:averaging}.

The asymptotic improvement is stronger: averaging suppresses the token-specific noise that count correction leaves inside $m_d^{-1}\sum_j e^{\beta X_{aj}}$. The interval of signals where averaging succeeds and count correction fails has width $\sqrt{2\alpha_1}(1-\sqrt\rho)$. This result assumes that the groups are supplied correctly and that the planted item is a singleton; it does not claim that arbitrary clustering or merging improves attention.

\paragraph{Hidden-singleton variant.} In the basic model, the planted item is the only singleton group, so its group size reveals its identity before any scores are examined. To remove this shortcut, add $L_d=\lceil e^{a_0d}\rceil$ independent $N(0,d)$ distractor scores, each in its own singleton group, where $0<a_0<\alpha_1$. Place the planted item uniformly at random among these $L_d+1$ singleton groups and hide its location. Define the two critical signal levels
\begin{equation}
 c_{\mathrm{avg},0}:=\max\{\sqrt{2\rho\alpha_1},\sqrt{2a_0}\}
 <c_{\mathrm{cnt}}:=\sqrt{2\alpha_1}.
 \label{eq:singletonfloor}
\end{equation}
\begin{corollary}[Hidden-singleton separation]\label{cor:singletons}
In this augmented model, count correction fails for every data-dependent choice of $\beta$ when $s<c_{\mathrm{cnt}}$, while the fixed choice $\beta=c_{\mathrm{cnt}}$ retrieves when $s>c_{\mathrm{cnt}}$. Averaging retrieves whenever $s>c_{\mathrm{avg},0}$ using any fixed $\beta\ge\max\{b_1,\sqrt{2a_0}\}$. Hence every $s\in(c_{\mathrm{avg},0},c_{\mathrm{cnt}})$ is retrieved by averaging and missed by count correction for every choice of $\beta$. If $a_0\le\rho\alpha_1$, then $c_{\mathrm{avg},0}=\sqrt{2\rho\alpha_1}$ and the original separation interval is unchanged. A method that observes only the partition, and not the scores, identifies the target with probability at most $1/(L_d+1)$.
\end{corollary}
\begin{proof}[Proof idea; complete proof in Appendix~\ref{app:averaging}]
The augmented distractor denominator is the sum of the original clustered contribution and the new singleton contribution. After taking $d^{-1}\log$, this sum has the larger of their two limiting values. Adding distractors can only decrease the planted mass, so the adaptive-temperature failure result for count correction remains valid. For averaging, the clustered representatives require signal $\sqrt{2\rho\alpha_1}$ and the singleton distractors require signal $\sqrt{2a_0}$; their maximum is $c_{\mathrm{avg},0}$. The partition-only bound follows because the hidden target is uniform among the $L_d+1$ singleton groups.
\end{proof}
The hidden-singleton construction removes identification from group size alone, but it still assumes that the target is a singleton rather than a member of a mixed relevant--irrelevant group.

\section{Practical meaning and scope}
The theory assumes that the groups are supplied; it does not infer them from the keys. Subject to that assumption, the conclusions are stable to small errors, but the two corrections have different computational and semantic consequences.

\paragraph{Imperfect group statistics.} Only relative count errors matter for count correction. If positive estimates satisfy
\begin{equation}
 \max_{i\ne *}\left|
 \log(\widehat m_i/m_i)-\log(\widehat m_*/m_*)
 \right|\le r_d,
 \label{eq:relativeerror}
\end{equation}
then using the estimated counts changes $\log[(1-A_*)/A_*]$, the log ratio of total distractor mass to planted mass, by at most $r_d$. In other words, if $D$ is this log ratio under the true counts and $\widehat D$ is the same quantity under the estimated counts, then $|\widehat D-D|\le r_d$. Likewise, if every averaged score has error at most $u_d$, the corresponding log ratio changes by at most $2\beta u_d$. Thus errors smaller than order $d$ do not change the fixed-$\beta$ conclusions away from the success--failure boundary. These bounds assume the groups are already available; they do not justify any particular clustering procedure. Proposition~\ref{prop:robust} gives the deterministic proof in Appendix~\ref{app:averaging}.

\paragraph{Cost and interpretation.} Count correction still scores all $N_d$ tokens, so its per-query cost remains $O(N_dd)$. Averaging can instead precompute $M_d$ centroids in $O(N_dd)$ time and then score only those centroids in $O(M_dd)$ time per query. If all values have norm at most $R$, target mass $A_*$ guarantees output error at most $2R(1-A_*)$. This makes successful mass retrieval sufficient for value recovery, but not necessary; Proposition~\ref{prop:values} gives the proof in Appendix~\ref{app:averaging}. It also explains an important limitation: averaging may erase the target when a supplied group mixes relevant and irrelevant items.

\section{How broadly does the result apply?}
The key result is a gap between the smallest planted signals required by the two methods: averaging succeeds above $\sqrt{2\rho\alpha_1}$, whereas count correction requires $\sqrt{2\alpha_1}$, as shown in \eqref{eq:gap}.

\paragraph{A fixed finite score array.} Before taking $d\to\infty$, consider one realized collection of distinct distractor scores $x_1,\ldots,x_{N_d}$ as fixed numbers. Their softmax weights at temperature $\beta$ are $p_i(\beta)=e^{\beta x_i}/\sum_j e^{\beta x_j}$, and their entropy $S(p)=-\sum_i p_i\log p_i$ measures how broadly the distractor attention is spread. Proposition~\ref{prop:entropy} shows that the best count-correction temperature is exactly the point where $S(p(\beta))=\log m_d$, or equivalently where the effective number $e^{S(p)}$ of attended distractors is $m_d$. This optimization identity is known as entropic value-at-risk (EVaR); it is an existing result used here as an oracle benchmark, not a new contribution.

\paragraph{Deeper and non-Gaussian hierarchies.} The two-level model has clusters and tokens, but the same comparison holds for any fixed number of Gaussian levels, such as source, document, passage, and token (Theorem~\ref{thm:hierarchy}). Gaussianity can also be relaxed when each level contributes centered, nondegenerate independent increments whose cumulant generating function is finite on $\mathbb R$ and has derivative tending to $\pm\infty$ at the two tails. For an increment $Y_\ell$, the cumulant generating function is $\Lambda_\ell(t)=\log\E e^{tY_\ell}$; it records the exponential-tail behavior that determines the retrieval threshold. For a Gaussian increment it is simply a quadratic function of $t$, recovering the square-root formulas above. Theorem~\ref{thm:cramer} proves the analogous count-versus-averaging gap for this non-Gaussian class. Full assumptions and proofs are in Appendices~\ref{app:entropy}--\ref{app:cramer}; the results do not cover bounded-support saturation regimes, arbitrary learned score distributions, or unknown groups.

\paragraph{Boundary cases and heterogeneous groups.} Appendix~\ref{app:extensions} records several extensions that delimit the basic comparison. Proposition~\ref{prop:partial} gives the optimal threshold when only a fraction $\gamma\log m_d$ of the multiplicity correction is subtracted. Proposition~\ref{prop:types} allows a fixed finite collection of distractor types with unequal cluster sizes and also separates retrieval of a relevant group from retrieval of one member. Proposition~\ref{prop:duplicates} characterizes the inverse-count rule through exact replication invariance, while Corollary~\ref{cor:neff} shows that, in the two-transition case, no single temperature-independent effective context length can reproduce the full free-energy curve. The same appendix also records the endpoint and conventional-scaling checks used to delimit the asymptotic model.

\section{Experiments and limitations}
\label{sec:experiments}
The Gaussian simulations check the formulas within the stated model: finite-array temperature optimization approaches the predicted threshold, and cluster and token participation approach their respective limits. Separate constrained optimizations reproduce the free energy across 120 parameter settings and check the finite-array, deeper-hierarchy, and non-Gaussian extensions. Figure~\ref{fig:checks} and the full protocols are in Appendix~\ref{app:numerics}. These synthetic checks concern the theory itself. The following frozen-encoder experiment instead provides a controlled illustration of the two operations on learned sentence scores and of failure modes outside the theorem's assumptions; it is not a test of the Gaussian asymptotic law.

\paragraph{Frozen sentence retrieval.} We use BANKING77, a collection of banking questions with 77 intent labels \citep{casanueva2020banking}, and the frozen \texttt{all-MiniLM-L6-v2} sentence encoder \citep{minilm2021model}. Training sentences supply the keys; disjoint subsets of the official test split supply 308 validation and 924 evaluation queries. Scores are cosine similarities. Each context contains all 77 intents, initially with one relevant key and 38 singleton distractors. The other 38 intents contribute four or eight same-intent keys, or eight exact copies. Labels supply the groups; same-intent sentences need not be paraphrases. Three seeded context draws are shared by all methods. These are controlled contexts, not the standard intent-classification benchmark.

Ordinary attention and count correction weight individual keys. Averaging uses one mean score and one mean value per group, without renormalizing the mean key. Each value indicates its key's intent, so the relevant coordinate of the output is the total relevant value mass, equal to $A_*$ when the target is a singleton. Each method independently selects a cosine-score multiplier from 81 values in $[0.25,1024]$ by minimizing validation negative log relevant value mass. Evaluation queries are not used for this choice. Full settings, uncertainty estimates, and additional metrics are in Appendix~\ref{app:learned}.

\begin{table}[t]
\centering
\small
\caption{Mean relevant value mass on 924 evaluation queries. Each method uses its own validation-selected temperature. The last column gives the paired averaging-minus-count difference with a 95\% intent-block bootstrap interval. Groups are supplied using dataset labels.}
\label{tab:learned}
\begin{tabular}{lrrrr}
\toprule
Context & Ordinary & Count & Average & Difference [95\% interval] \\
\midrule
One key per intent & 0.379 & 0.379 & 0.379 & $+0.000$ [$+0.000$, $+0.000$] \\
Exact copies (8) & 0.288 & 0.379 & 0.379 & $+0.000$ [$+0.000$, $+0.000$] \\
Same intent (4) & 0.267 & 0.345 & 0.406 & $+0.062$ [$+0.056$, $+0.067$] \\
Same intent (8) & 0.234 & 0.342 & 0.440 & $+0.098$ [$+0.088$, $+0.108$] \\
Four relevant keys & 0.614 & 0.626 & 0.465 & $-0.161$ [$-0.181$, $-0.142$] \\
Mixed relevant group & 0.234 & 0.243 & 0.005 & $-0.238$ [$-0.269$, $-0.208$] \\
\bottomrule
\end{tabular}
\end{table}

\paragraph{Multiplicity and score variation.} Exact copies make count correction and averaging identical, as expected. With eight distinct keys per distractor group, their mean relevant masses are instead $0.342$ and $0.440$ after separate temperature selection (Table~\ref{tab:learned}); the paired gain is $0.098$ with 95\% interval $[0.088,0.108]$. Selecting one random representative per group gives $0.375$, so retaining only one key does not explain the full gain. However, a label-aware oracle that chooses the best count-correction temperature separately for each evaluation query attains $0.511$. Thus the learned-score experiment supports the distinction between the operations, but does not establish the theorem's all-temperature separation for these sentence scores.

\paragraph{When averaging loses relevant evidence.} Giving the relevant intent four distinct keys reverses the comparison: count correction achieves $0.626$ relevant mass and averaging $0.465$. Variation within a relevant group can contain useful matches to the query. Moving three distractor keys into the singleton's group is more damaging: averaging reduces relevant mass to $0.005$, versus $0.243$ for count correction. Here averaging also dilutes the relevant value to one quarter of the group's value. These cases are outside the singleton-target guarantee and show why group membership alone does not justify averaging.

\paragraph{What remains outside the theory.} The model assumes exponentially many distractors, supplied groups, and a singleton target. Its exact phase diagram assumes Gaussian scores, while the non-Gaussian result still requires centered, nondegenerate independent increments with finite all-real exponential moments and steep cumulants. The temperature scaling also differs from conventional fixed scaled-dot-product attention. Our sentence experiment uses bounded cosine scores and label-supplied groups; it does not validate these asymptotic assumptions or a grouping algorithm. Appendix~\ref{app:transformer} reports a causal intervention on pre-softmax self-attention logits in a frozen DistilBERT model. It confirms that averaging suppresses within-group score extremes in actual attention, but its downstream effects are small, metric-dependent, and position-sensitive in a task on which the model performs near chance; it therefore does not support a general performance advantage.

Count correction removes multiplicity while retaining variation within a group. Averaging removes that variation as well. The theory shows why this yields a strict retrieval advantage against grouped distractors; the experiment shows why the same operation can remove useful evidence from a relevant group.

\section*{AI use statement}
In this work, the author used generative AI tools to assist with writing mathematical proofs, proposing and refining hypotheses, designing and evaluating the research methodology and numerical experiments, implementing the methods and verification code. The author also used generative AI tools to identify and check relevant literature and to improve the manuscript's clarity, structure, and formatting. The author takes full responsibility for the final content of the paper, including all text, claims, proofs, code, and artifacts produced with the assistance of generative AI.
\bibliography{bib}
\bibliographystyle{amsplain}
\clearpage
\appendix
\section{Free energy and Gaussian comparison: complete proofs}
\label{app:free}
Throughout this appendix, $d\to\infty$, $\alpha_1,\alpha_2>0$, and $\alpha=\alpha_1+\alpha_2$. We use $M_d=\lceil e^{\alpha_1d}\rceil$ clusters and $m_d=\lceil e^{\alpha_2d}\rceil$ tokens per cluster, as in \eqref{eq:population}; hence $N_d=M_dm_d$ and $d^{-1}\log M_d\to\alpha_1$, $d^{-1}\log m_d\to\alpha_2$, and $d^{-1}\log N_d\to\alpha$. The ceiling functions do not affect any limit. Unless stated otherwise, inverse temperatures are fixed and positive, and $\pto$ denotes convergence in probability.

\subsection{Counting states subject to the constraints}
\begin{lemma}[Two-level counting principle]\label{lem:counting} Let $a=1,\ldots,M_d$ index clusters and $j=1,\ldots,m_d$ index tokens within each cluster. Let $U_a\sim N(0,v_1d)$ and $V_{aj}\sim N(0,v_2d)$ be mutually independent, where $v_1,v_2>0$. For $X_{aj}=U_a+V_{aj}$, the limiting free energy is
\begin{equation}
 \max_{\substack{x^2/(2v_1)\le\alpha_1\\x^2/(2v_1)+y^2/(2v_2)\le\alpha}}
 \left\{\alpha-\frac{x^2}{2v_1}-\frac{y^2}{2v_2}+\beta(x+y)\right\}.
 \label{eq:generalvariational}
\end{equation}
Convergence holds in probability and in $L^1$.
\end{lemma}
\begin{proof}
Write $I_1(x)=x^2/(2v_1)$ and $I_2(y)=y^2/(2v_2)$. For a fixed interval $B$ of positive length,
\begin{equation}
 \frac1d\log\Prob(U_a/d\in B)\xrightarrow{d\to\infty}-\inf_{x\in B}I_1(x).
 \label{eq:gaussianldp}
\end{equation}
Since $U_a/d\sim N(0,v_1/d)$, its standard deviation is $\sqrt{v_1/d}$ and its density is $\sqrt{d/(2\pi v_1)}\exp[-dx^2/(2v_1)]$. Integrating this density gives the displayed limit; using interior subintervals establishes the lower bound. The identical statement holds for $V_{aj}/d$.

\emph{Upper bound.} Choose $R$ large enough that $R^2/(2v_1)>\alpha_1+1$ and $R^2/(2v_2)>\alpha+1$. A union bound makes the probability that any $U_a/d$ or $V_{aj}/d$ leaves $[-R,R]$ tend to zero exponentially. Partition this compact interval into finitely many bins of width at most $\delta$. For a pair of bins $B,C$, put $r_1=\inf_B I_1$ and $r_2=\inf_C I_2$, and define the finite index counts
\begin{equation*}
 K_B:=\#\{a:U_a/d\in B\},\qquad
 T_{BC}:=\#\{(a,j):U_a/d\in B,\ V_{aj}/d\in C\}.
\end{equation*}
Thus $K_B$ counts clusters with $U_a/d\in B$, while $T_{BC}$ counts cluster-token index pairs satisfying both bin conditions; it does not count the continuum of points in $B\times C$. By \eqref{eq:gaussianldp},
\begin{equation*}
 \E K_B=e^{d(\alpha_1-r_1+o(1))},\qquad
 \E T_{BC}=e^{d(\alpha-r_1-r_2+o(1))}.
\end{equation*}
Fix $\eta>0$. Markov's inequality, followed by a union bound over the finite bins, shows with probability tending to one that bins with $r_1>\alpha_1+\eta$ or $r_1+r_2>\alpha+\eta$ are empty, and all other bins satisfy
\begin{equation*}
 T_{BC}\le\exp\{d[\max(\alpha-r_1-r_2,0)+\eta]\}.
\end{equation*}
We next convert the count bound into a bound on the partition function $\Z_d=\sum_{a,j}e^{\beta X_{aj}}$. For each bin pair define its contribution
\begin{equation*}
 \Z_{BC}:=\sum_{\substack{(a,j):\,U_a/d\in B\\V_{aj}/d\in C}}e^{\beta X_{aj}}.
\end{equation*}
There are $T_{BC}$ summands in $\Z_{BC}$. For every one of them, $X_{aj}/d=U_a/d+V_{aj}/d$ is at most $\sup_{(x,y)\in B\times C}(x+y)$. Hence
\begin{equation}
 \Z_{BC}\le T_{BC}\exp\left\{\beta d\sup_{(x,y)\in B\times C}(x+y)\right\}.
 \label{eq:bincontribution}
\end{equation}
The bounds above control $T_{BC}$, and substituting them into \eqref{eq:bincontribution} bounds each bin pair's contribution. Since the bin pairs partition all index pairs $(a,j)$ on the event under consideration, $\Z_d=\sum_{B,C}\Z_{BC}$. For fixed $\delta$, this is a sum over finitely many bin pairs, independently of $d$. Thus replacing the largest bin-pair bound by their sum changes $L_d=d^{-1}\log\Z_d$ by at most $(\log\text{number of bin pairs})/d=o(1)$.

First let $d\to\infty$, and then send $\eta\downarrow0$. The relaxed constraints used above,
\begin{equation*}
 r_1\le\alpha_1+\eta,\qquad r_1+r_2\le\alpha+\eta,
\end{equation*}
therefore reduce to the stated constraints $r_1\le\alpha_1$ and $r_1+r_2\le\alpha$. On this limiting set, $\alpha-r_1-r_2\ge0$, so $\max(\alpha-r_1-r_2,0)=\alpha-r_1-r_2$. All these sets lie in the compact square $[-R,R]^2$, and the objective is continuous, so their maxima converge to the maximum in \eqref{eq:generalvariational}. Finally send $\delta\downarrow0$.

\emph{Lower bound.} Fix $(x,y)$ satisfying both constraints strictly. Choose small intervals $B,C$ around these coordinates, and recall that $r_1=\inf_{u\in B}I_1(u)$ and $r_2=\inf_{v\in C}I_2(v)$. By choosing the intervals sufficiently small, the strict feasibility of $(x,y)$ gives $r_1<\alpha_1$ and $r_1+r_2<\alpha$.

Write $K_B=\sum_{a=1}^{M_d}\mathbf1\{U_a/d\in B\}$. It is binomial with mean $\E K_B=M_d\Prob(U_a/d\in B)=\exp\{d(\alpha_1-r_1+o(1))\}$, which grows exponentially because $\alpha_1-r_1>0$. Its relative error $|K_B-\E K_B|/\E K_B$ tends to zero since $\Var(K_B)/(\E K_B)^2\le1/\E K_B\to0$, Chebyshev's inequality gives $K_B/\E K_B\pto1$; in particular $d^{-1}\log K_B\pto\alpha_1-r_1$. Conditional on the cluster-level variables $(U_a)_a$, the $K_Bm_d$ indicators for $V_{aj}/d\in C$ are independent Bernoulli variables. Their conditional total mean has exponential rate $\alpha-r_1-r_2>0$. A binomial Chernoff bound therefore gives
\begin{equation*}
 \frac1d\log T_{BC}\pto\alpha-r_1-r_2.
\end{equation*}
These second-level trials remain independent even when several tokens belong to the same selected cluster, because selection used only the $U_a$. No condition $r_2<\alpha_2$ is required. Every index pair counted by $T_{BC}$ satisfies $X_{aj}/d=U_a/d+V_{aj}/d\ge\inf_{(u,v)\in B\times C}(u+v)$. Its exponential weight is therefore at least $\exp\{\beta d\inf_{B\times C}(u+v)\}$, and hence
\begin{equation*}
 \Z_d\ge T_{BC}\exp\left\{\beta d\inf_{(u,v)\in B\times C}(u+v)\right\}.
\end{equation*}
Combining this inequality with $d^{-1}\log T_{BC}\pto\alpha-r_1-r_2$ gives, with probability tending to one,
\begin{equation*}
 \liminf_{d\to\infty}L_d\ge
 \alpha-r_1-r_2+\beta\inf_{(u,v)\in B\times C}(u+v).
\end{equation*}
Now shrink $B$ and $C$ to the chosen point $(x,y)$. Continuity of $I_1$, $I_2$, and $u+v$ makes the right-hand side tend to $\alpha-I_1(x)-I_2(y)+\beta(x+y)$, proving the lower bound at every strictly feasible $(x,y)$.

It remains to include points on the boundary of the feasible set. If $(x,y)$ satisfies the non-strict constraints, then for every $0<t<1$ the point $(tx,ty)$ satisfies them strictly: its two costs are $t^2I_1(x)<\alpha_1$ and $t^2[I_1(x)+I_2(y)]<\alpha$, because $\alpha_1,\alpha>0$. Apply the strict-point lower bound to $(tx,ty)$ and let $t\uparrow1$. The objective is continuous, so its value converges to the value at $(x,y)$. Since the feasible set is compact and the objective is continuous, it has a maximizer; applying this argument to that maximizer proves the lower bound for the maximum.

\emph{$L^1$ convergence.} We have already proved $L_d\pto F$. To strengthen this to $L^1$ convergence, we must show
\begin{equation}
 \E|L_d-F|\longrightarrow0.
 \label{eq:l1definition}
\end{equation}
It is enough to show that $(L_d)_d$ is uniformly integrable, meaning that large values make a uniformly negligible contribution to its expectation:
\begin{equation}
 \lim_{K\to\infty}\sup_d\E\big[|L_d|\mathbf1\{|L_d|>K\}\big]=0.
 \label{eq:uniformintegrability}
\end{equation}
The elementary inequality
\begin{equation*}
 |L_d|\le d^{-1}\log N_d+\beta d^{-1}\max_{a,j}|X_{aj}|
\end{equation*}
reduces this to controlling the largest score. Each $X_{aj}$ is $N(0,(v_1+v_2)d)$, so a union bound gives
\begin{equation*}
 \Prob\left(\frac{\max_{a,j}|X_{aj}|}{d}>u\right)
 \le2N_d\exp\left[-\frac{du^2}{2(v_1+v_2)}\right].
\end{equation*}
Because $d^{-1}\log N_d\to\alpha$, the right-hand side decays exponentially in $d$ for every sufficiently large fixed $u$. Integrating this tail bound shows
\begin{equation*}
 \sup_d\E\left[\left(d^{-1}\max_{a,j}|X_{aj}|\right)^2\right]<\infty.
\end{equation*}
Since $d^{-1}\log N_d$ is also uniformly bounded, the displayed bound on $L_d$ gives $\sup_d\E L_d^2<\infty$. This second-moment bound implies \eqref{eq:uniformintegrability}, because
\begin{equation*}
 \E\big[|L_d|\mathbf1\{|L_d|>K\}\big]\le K^{-1}\E L_d^2.
\end{equation*}
Finally let $Y_d=L_d-F$. The same uniform-integrability property holds for $(Y_d)_d$. For any $\varepsilon>0$ and $K>\varepsilon$,
\begin{equation*}
 \E|Y_d|\le\varepsilon+K\Prob(|Y_d|>\varepsilon)+\E\big[|Y_d|\mathbf1\{|Y_d|>K\}\big].
\end{equation*}
First choose $K$ so that the last term is uniformly small, then let $d\to\infty$ so that the middle term vanishes by $Y_d\pto0$, and finally send $\varepsilon\downarrow0$. This proves \eqref{eq:l1definition}.
\end{proof}

\subsection{Evaluation of the variational formula}
\begin{proof}[Proof of Theorem~\ref{thm:free}]
Apply Lemma~\ref{lem:counting} with $v_1=\rho,v_2=1-\rho$. It remains to solve its strictly concave optimization problem.

The two constraints depend on $x^2$ and $y^2$, whereas $\beta(x+y)$ increases when either negative coordinate is replaced by its absolute value. Thus a maximizer can be chosen with $x,y\ge0$. If $\rho\le\alpha_1/\alpha$, temporarily ignore the constraint $x^2/(2\rho)\le\alpha_1$. Define
\begin{equation*}
 g(x,y):=\alpha-\frac{x^2}{2\rho}-\frac{y^2}{2(1-\rho)}+\beta(x+y),
 \qquad
 \frac{x^2}{2\rho}+\frac{y^2}{2(1-\rho)}\le\alpha.
\end{equation*}
For every $(x,y)$ satisfying the remaining constraint, Cauchy--Schwarz gives
\begin{equation*}
 x+y\le\left(\frac{x^2}{\rho}+\frac{y^2}{1-\rho}\right)^{1/2},
\end{equation*}
with equality exactly when $x/\rho=y/(1-\rho)$. Thus, for each fixed value of $x^2/(2\rho)+y^2/[2(1-\rho)]$, the largest possible value of $g$ occurs at $x=\rho t$, $y=(1-\rho)t$. The remaining constraint becomes $0\le t\le\sqrt{2\alpha}$, and on this line
\begin{equation*}
 g(\rho t,(1-\rho)t)=\alpha-\frac{t^2}{2}+\beta t.
\end{equation*}
The derivative of this one-variable function is $\beta-t$, so its maximum on $[0,\sqrt{2\alpha}]$ occurs at $t=\min(\beta,\sqrt{2\alpha})$. Hence
\begin{equation*}
 (x,y)=(\rho t,(1-\rho)t),\qquad t=\min(\beta,\sqrt{2\alpha}).
\end{equation*}
For this point, $x^2/(2\rho)=\rho t^2/2\le\rho\alpha\le\alpha_1$, so it also satisfies the constraint that was temporarily omitted. Substitution gives $\frem(\beta;\alpha,1)$.

Put $b_1=\sqrt{2\alpha_1/\rho}$ and $b_2=\sqrt{2\alpha_2/(1-\rho)}$. If $\rho>\alpha_1/\alpha$, then $b_1<b_2$. For $\beta\le b_1$, the point $(x,y)=(\beta\rho,\beta(1-\rho))$ satisfies
\begin{equation*}
 \frac{x^2}{2\rho}=\frac{\beta^2\rho}{2}\le\alpha_1,
 \qquad
 \frac{x^2}{2\rho}+\frac{y^2}{2(1-\rho)}=\frac{\beta^2}{2}\le\frac{\alpha_1}{\rho}<\alpha,
\end{equation*}
and the two equations $\partial g/\partial x=\partial g/\partial y=0$ show that it is the maximizer. For $b_1\le\beta\le b_2$, take
\begin{equation*}
 x=\sqrt{2\rho\alpha_1},\qquad y=\beta(1-\rho).
\end{equation*}
Here $x^2/(2\rho)=\alpha_1$ and
\begin{equation*}
 \frac{x^2}{2\rho}+\frac{y^2}{2(1-\rho)}
 =\alpha_1+\frac{\beta^2(1-\rho)}{2}\le\alpha_1+\alpha_2=\alpha.
\end{equation*}
With $\lambda_1=\beta/b_1-1\ge0$ and $\lambda_2=0$, the derivative equations for the two constraints are
\begin{equation}
 \beta-\frac{(1+\lambda_1+\lambda_2)x}{\rho}=0,
 \qquad
 \beta-\frac{(1+\lambda_2)y}{1-\rho}=0.
 \label{eq:splitkkt}
\end{equation}
For $\beta\ge b_2$, choose
\begin{equation*}
 x=\sqrt{2\rho\alpha_1},\qquad y=\sqrt{2(1-\rho)\alpha_2}.
\end{equation*}
This point satisfies $x^2/(2\rho)=\alpha_1$ and $x^2/(2\rho)+y^2/[2(1-\rho)]=\alpha$. The choices $\lambda_2=\beta/b_2-1$ and $\lambda_1=\beta/b_1-\beta/b_2$ are nonnegative and satisfy \eqref{eq:splitkkt}. Because $g$ is concave and both constraint functions are convex, these equations together with the displayed inequalities identify the global maximizer. Substitution gives \eqref{eq:split}, with the formulas agreeing at all boundaries.

Finally, write $D_d=\log\Z_d-\beta X_*$. Then $D_d/d\pto F-\beta s$ and $A_*=(1+e^{D_d})^{-1}$. A strictly negative or positive limiting exponent proves the two retrieval conclusions without any independence requirement on $X_*$.
\end{proof}

\subsection{Finite-size concentration and retrieval margins}
The free-energy theorem gives a limiting statement. The following proposition is a separate finite-$d$ refinement.
\begin{proposition}[Gaussian concentration]\label{prop:concentration}
For every $u>0$,
\begin{equation}
 \Prob\{|L_d-\E L_d|\ge u\}\le2\exp[-du^2/(2\beta^2)].
 \label{eq:concentration}
\end{equation}
Let $e_d=|\E L_d-F|\to0$ and $t_d(r)=\Prob(|X_*/d-s|>r)$. Except on an event of probability at most $2e^{-du^2/(2\beta^2)}+t_d(r)$,
\begin{equation}
 \left|d^{-1}\log\frac{1-A_*}{A_*}-(F-\beta s)\right|
 \le e_d+u+\beta r.\label{eq:margin}
\end{equation}
Thus a positive margin $\beta s-F>e_d+u+\beta r$ implies $1-A_*\le\exp[-d(\beta s-F-e_d-u-\beta r)]$ on this event.
\end{proposition}
\begin{proof}
We use the standard Gaussian Lipschitz concentration inequality, see~\citet{boucheron2013concentration}: if $W$ is a vector of independent $N(0,1)$ variables and $f(W)$ is $c$-Lipschitz in Euclidean distance, then
\begin{equation*}
 \Prob\{|f(W)-\E f(W)|\ge u\}\le2\exp\left[-\frac{u^2}{2c^2}\right].
\end{equation*}
Thus, to prove \eqref{eq:concentration}, it suffices to show that $L_d$ is $(\beta/\sqrt d)$-Lipschitz as a function of the independent standard Gaussian variables $(G,H)$. The derivatives of $\log\Z_d$ are
\begin{equation*}
 \partial_{G_a}\log\Z_d=\beta\sqrt{\rho d}\,P_a,
 \qquad
 \partial_{H_{aj}}\log\Z_d=\beta\sqrt{(1-\rho)d}\,p_{aj}.
\end{equation*}
Therefore
\begin{equation*}
 \sum_a(\partial_{G_a}\log\Z_d)^2+
 \sum_{a,j}(\partial_{H_{aj}}\log\Z_d)^2
 =\beta^2d\left[\rho\sum_aP_a^2+(1-\rho)\sum_{a,j}p_{aj}^2\right]
 \le\beta^2d.
\end{equation*}
Hence, we obtain $\|\nabla\log\Z_d\|_2\le\beta\sqrt d$. Since $L_d=d^{-1}\log\Z_d$, its gradient has Euclidean norm at most $\beta/\sqrt d$, so $L_d$ is $(\beta/\sqrt d)$-Lipschitz.

On the concentration event, $|L_d-F|\le |L_d-\E L_d|+|\E L_d-F|\le u+e_d$. Intersect it with the event $|X_*/d-s|\le r$ and use $d^{-1}\log[(1-A_*)/A_*]=L_d-\beta X_*/d$ to obtain \eqref{eq:margin}. If $\Delta_d:=\beta s-F-e_d-u-\beta r>0$, then \eqref{eq:margin} gives
\begin{equation*}
 \frac1d\log\frac{1-A_*}{A_*}\le-\Delta_d,
 \qquad
 1-A_*=A_*\frac{1-A_*}{A_*}\le e^{-d\Delta_d},
\end{equation*}
which proves the final assertion of the proposition. No explicit rate is asserted for the deterministic bias $e_d$; this statement must not be read as a rate to the limiting threshold without that bias term.
\end{proof}

\subsection{Gaussian interpolation}
\begin{lemma}[Covariance differentiation]\label{lem:interpolation}
For a differentiable positive-definite covariance path $C(t)$ and a centered Gaussian vector $Y(t)~\sim N(0,C(t))$, the log-sum-exp $g(y)=\log\sum_ie^{\beta y_i}$ satisfies
\begin{equation}
 \frac{d}{dt}\E g(Y(t))
 =\frac12\sum_{i,j}C'_{ij}(t)\E\partial_i\partial_jg(Y(t)).
 \label{eq:interpolation}
\end{equation}
\end{lemma}
\begin{proof}
Let $\phi_t$ be the density of $N(0,C(t))$. Its Fourier transform is
\begin{equation*}
 \widehat\phi_t(\xi)=\exp\left[-\frac12\xi^\top C(t)\xi\right].
\end{equation*}
Differentiating in $t$ gives
\begin{equation*}
 \frac{d}{dt}\widehat\phi_t(\xi)
 =-\frac12\sum_{i,j}C'_{ij}(t)\xi_i\xi_j\widehat\phi_t(\xi)
 =\widehat{\left[\frac12\sum_{i,j}C'_{ij}(t)\partial_i\partial_j\phi_t\right]}(\xi).
\end{equation*}
Therefore
\begin{equation}
 \partial_t\phi_t(y)=\frac12\sum_{i,j}C'_{ij}(t)\partial_i\partial_j\phi_t(y).
 \label{eq:densitycovariance}
\end{equation}
Writing $p_i(y)=e^{\beta y_i}/\sum_ke^{\beta y_k}$, the first two derivatives of $g$ are
\begin{equation}
 \partial_i g(y)=\beta p_i(y),
 \qquad
 \partial_i\partial_jg(y)=\beta^2\big[p_i(y)\mathbf1_{i=j}-p_i(y)p_j(y)\big].
 \label{eq:logsumexpderivatives}
\end{equation}
In particular, the second derivatives are bounded. Hence differentiating under the integral and using \eqref{eq:densitycovariance} yields
\begin{align*}
 \frac{d}{dt}\E g(Y(t))
 &=\int g(y)\,\partial_t\phi_t(y)\,dy \notag\\
 &=\frac12\sum_{i,j}C'_{ij}(t)
   \int g(y)\,\partial_i\partial_j\phi_t(y)\,dy \notag\\
 &=\frac12\sum_{i,j}C'_{ij}(t)
   \int \partial_i\partial_jg(y)\,\phi_t(y)\,dy,
 \end{align*}
where the last equality is integration by parts twice. Since $g(y)\le\log n+\beta\max_i|y_i|$ and Gaussian densities and their derivatives have Gaussian tails, the differentiations and the boundary terms are justified. This is \eqref{eq:interpolation}. Singular endpoints are obtained by continuity; only positive-definite $C(t)$ is needed in the displayed calculation.
\end{proof}
\begin{proof}[Proof of Theorem~\ref{thm:comparison}]
Index the Gaussian score vector by $i=(a,j)$. Its covariance matrix satisfies
\begin{equation*}
 \frac{d}{d\rho}\Cov(X_{aj},X_{bk})=
 \begin{cases}
 d,&a=b,\ j\ne k,\\
 0,&\text{otherwise}.
 \end{cases}
\end{equation*}
Indeed, every score has variance $d$, so the derivative is $0$ when $(a,j)=(b,k)$; two distinct scores have covariance $\rho d$ exactly when they belong to the same cluster. For $g(x)=\log\sum_{a,j}e^{\beta x_{aj}}$, \eqref{eq:logsumexpderivatives} gives, for $j\ne k$,
\begin{equation*}
 \partial_{aj}\partial_{ak}g=-\beta^2p_{aj}p_{ak}.
\end{equation*}
Applying \eqref{eq:interpolation} and dividing by $d$ gives
\begin{equation*}
 \frac{d}{d\rho}\E L_d
 =\frac{1}{2d}\sum_a\sum_{j\ne k}d\,\E[-\beta^2p_{aj}p_{ak}]
 =-\frac{\beta^2}{2}\E\sum_a\sum_{j\ne k}p_{aj}p_{ak}.
\end{equation*}
For each cluster,
\begin{equation*}
 \sum_{j\ne k}p_{aj}p_{ak}
 =\left(\sum_jp_{aj}\right)^2-\sum_jp_{aj}^2.
\end{equation*}
Let $\pc=\sum_a P_a^2$ and $\pt=\sum_{a,j}p_{aj}^2$. Summing over $a$ gives $\sum_a\sum_{j\ne k}p_{aj}p_{ak}=\pc-\pt$, and therefore
\begin{equation*}
 \frac{d}{d\rho}\E L_d
 =-\frac{\beta^2}{2}\E(\pc-\pt)\le0,
\end{equation*}
because $\pc-\pt=\sum_a\sum_{j\ne k}p_{aj}p_{ak}\ge0$. Thus, for $0<\rho_1<\rho_2<1$,
\begin{equation*}
 \E L_d(\rho_2)-\E L_d(\rho_1)
 =-\frac{\beta^2}{2}\int_{\rho_1}^{\rho_2}\E(\pc-\pt)\,d\rho\le0.
\end{equation*}
Continuity extends this inequality to $\rho_1=0$ and $\rho_2=1$. Finally, $L^1$ convergence gives $\E L_d(\rho)\to F(\beta;\rho)$ for every fixed $\rho$, so $F(\beta;\rho)\le F(\beta;0)=\frem(\beta;\alpha,1)$. If $\beta s>\frem(\beta;\alpha,1)$, then $\beta s>F(\beta;\rho)$, and Theorem~\ref{thm:free} gives $A_*\pto1$.
\end{proof}

\section{Participation without an incorrect winner-take-all claim}
\label{app:participation}
\begin{proposition}[Expected participation]\label{prop:participation}
Equation~\eqref{eq:participation} holds for every fixed $\beta>0$ and $0<\rho<1$. A strictly positive limiting participation expectation does not assert a deterministic limiting participation value.
\end{proposition}
\begin{proof}
Temporarily replace $\rho$ and $1-\rho$ by independent positive variances $v_1,v_2$, and write
\begin{equation*}
 X_{aj}^{(v)}
 =\sqrt{v_1d}\,G_a+\sqrt{v_2d}\,H_{aj},
 \qquad
 L_d(v_1,v_2)
 =\frac1d\log\sum_{a,j}e^{\beta X_{aj}^{(v)}}.
\end{equation*}
The covariance is
\begin{equation*}
 \Cov(X_{aj}^{(v)},X_{bk}^{(v)})
 =d\left(v_1\mathbf1_{a=b}
 +v_2\mathbf1_{(a,j)=(b,k)}\right).
\end{equation*}
Consequently,
\begin{equation}
 \partial_{v_1}\Cov(X_{aj}^{(v)},X_{bk}^{(v)})
 =d\mathbf1_{a=b},
 \qquad
 \partial_{v_2}\Cov(X_{aj}^{(v)},X_{bk}^{(v)})
 =d\mathbf1_{(a,j)=(b,k)}.
 \label{eq:variancecovariancederivatives}
\end{equation}

Let
\begin{equation*}
 g(x)=\log\sum_{a,j}e^{\beta x_{aj}}.
\end{equation*}
By \eqref{eq:logsumexpderivatives},
\begin{equation*}
 \partial_{aj}\partial_{bk}g(x)
 =\beta^2\left(
 p_{aj}\mathbf1_{(a,j)=(b,k)}-p_{aj}p_{bk}
 \right),
\end{equation*}
where $\partial_{aj}$ means differentiation with respect to the score coordinate $x_{aj}$. Applying Lemma~\ref{lem:interpolation}, using $L_d=d^{-1}g$, and substituting the first identity in \eqref{eq:variancecovariancederivatives} gives
\begin{align}
 \partial_{v_1}\E L_d
 &=\frac1{2d}\sum_{a,j}\sum_{b,k}
   d\mathbf1_{a=b}\,
   \E[\partial_{aj}\partial_{bk}g]
 =\frac{\beta^2}{2}\E\left[
   \sum_{a,j}p_{aj}
   -\sum_a\sum_{j,k}p_{aj}p_{ak}
   \right] \notag\\
 &=\frac{\beta^2}{2}\E\left[
   1-\sum_a\left(\sum_jp_{aj}\right)^2
   \right]
  =\frac{\beta^2}{2}(1-\E\pc).
 \label{eq:clusterparticipationderivative}
\end{align}
Likewise, the second identity in \eqref{eq:variancecovariancederivatives} gives
\begin{align}
 \partial_{v_2}\E L_d
 &=\frac1{2d}\sum_{a,j}\sum_{b,k}
   d\mathbf1_{(a,j)=(b,k)}\,
   \E[\partial_{aj}\partial_{bk}g]
   =\frac{\beta^2}{2}\E\left[
   \sum_{a,j}p_{aj}-\sum_{a,j}p_{aj}^2
   \right] \notag\\
 &=\frac{\beta^2}{2}(1-\E\pt).
 \label{eq:tokenparticipationderivative}
\end{align}
Equations~\eqref{eq:clusterparticipationderivative} and
\eqref{eq:tokenparticipationderivative} are exact for every finite $d$.

It remains to justify passing these derivatives to the limit $d\to\infty$. Parameterize
\begin{equation*}
 v_i=t_i^2,\qquad t_i>0,
\end{equation*}
and define
\begin{equation*}
 \Phi_d(t_1,t_2)
 :=\E\left[
 \frac1d\log\sum_{a,j}
 \exp\{\beta\sqrt d(t_1G_a+t_2H_{aj})\}
 \right].
\end{equation*}
For every realized Gaussian array, the expression inside the expectation is convex, and hence $\Phi_d$ is convex. Lemma~\ref{lem:counting} and its $L^1$ conclusion give pointwise convergence
\begin{equation*}
 \Phi_d(t_1,t_2)\longrightarrow\Phi(t_1,t_2)
 :=F(\beta;t_1^2,t_2^2).
\end{equation*}

Fix one coordinate, hold the other coordinate constant, and write the resulting one-variable functions as $\Phi_d(t)$ and $\Phi(t)$. For every $h>0$ small enough, convexity gives
\begin{equation*}
 \frac{\Phi_d(t)-\Phi_d(t-h)}{h}
 \le \Phi_d'(t)
 \le
 \frac{\Phi_d(t+h)-\Phi_d(t)}{h}.
\end{equation*}
Taking $d\to\infty$ in the two secant slopes yields
\begin{equation*}
 \frac{\Phi(t)-\Phi(t-h)}{h}
 \le \liminf_{d\to\infty}\Phi_d'(t)
 \le \limsup_{d\to\infty}\Phi_d'(t)
 \le
 \frac{\Phi(t+h)-\Phi(t)}{h}.
\end{equation*}
If $\Phi$ is differentiable at $t$, sending $h\downarrow0$ makes both outer terms converge to $\Phi'(t)$. Therefore
\begin{equation}
 \Phi_d'(t)\longrightarrow\Phi'(t).
 \label{eq:convexderivativeconvergence}
\end{equation}
Applying this argument separately to $t_1$ and $t_2$, and using
\begin{equation*}
 \partial_{v_i}
 =\frac1{2t_i}\partial_{t_i},
 \qquad v_i=t_i^2,
\end{equation*}
shows that the derivatives in
\eqref{eq:clusterparticipationderivative} and
\eqref{eq:tokenparticipationderivative} converge to the corresponding derivatives of the limiting free energy.

The general-variance limiting free energy is
\begin{equation*}
 F(\beta;v_1,v_2)=\begin{cases}
 f(\beta;\alpha_1,v_1)+f(\beta;\alpha_2,v_2),&\alpha_1/v_1<\alpha_2/v_2,\\
 f(\beta;\alpha,v_1+v_2),&\alpha_1/v_1\ge\alpha_2/v_2.
 \end{cases}
\end{equation*}
For the one-level free energy,
\begin{equation}
 \partial_v f(\beta;a,v)
 =
 \begin{cases}
 \beta^2/2,&\beta\le\sqrt{2a/v},\\[2pt]
 \beta\sqrt{a/(2v)},&\beta\ge\sqrt{2a/v}.
 \end{cases}
 \label{eq:remvariancederivative}
\end{equation}
The two expressions agree when $\beta=\sqrt{2a/v}$. The derivatives of the two branches of $F(\beta;v_1,v_2)$ also agree when $\alpha_1/v_1=\alpha_2/v_2$, so the differentiability required in \eqref{eq:convexderivativeconvergence} holds.

From \eqref{eq:clusterparticipationderivative} and
\eqref{eq:tokenparticipationderivative},
\begin{equation}
 \lim_{d\to\infty}\E\pc
 =1-\frac{2}{\beta^2}\partial_{v_1}F(\beta;v_1,v_2),
 \qquad
 \lim_{d\to\infty}\E\pt
 =1-\frac{2}{\beta^2}\partial_{v_2}F(\beta;v_1,v_2).
 \label{eq:participationfromfreeenergy}
\end{equation}
Set $v_1=\rho$ and $v_2=1-\rho$. If $b_1<b_2$, then the first branch of the general-variance formula applies, and \eqref{eq:remvariancederivative} gives
\begin{equation*}
 \lim_{d\to\infty}\E\pc
 =\left(1-\frac{b_1}{\beta}\right)_+,
 \qquad
 \lim_{d\to\infty}\E\pt
 =\left(1-\frac{b_2}{\beta}\right)_+.
\end{equation*}
If $b_1\ge b_2$, then
\begin{equation*}
 F(\beta;v_1,v_2)
 =f(\beta;\alpha,v_1+v_2).
\end{equation*}
At $v_1+v_2=1$, both partial derivatives in
\eqref{eq:participationfromfreeenergy} equal the variance derivative in
\eqref{eq:remvariancederivative} with $a=\alpha$ and $v=1$. Hence
\begin{equation*}
 \lim_{d\to\infty}\E\pc
 =\lim_{d\to\infty}\E\pt
 =\left(1-\frac{b}{\beta}\right)_+.
\end{equation*}
These are exactly the two cases in \eqref{eq:participation}.
\end{proof}

For a probability vector $w$, $(\max_iw_i)^2\le\sum_iw_i^2\le\max_iw_i$. Thus vanishing expected token participation implies $\max_{a,j}p_{aj}\pto0$ by Markov's inequality. If the maximum mass tended to one in probability, its participation would tend to one in probability and, being bounded, in expectation. Formula~\eqref{eq:participation} rules that out for both tokens and clusters at finite $\beta$. The positive cluster-participation limit in the intermediate regime implies a nonvanishing probability of a macroscopic cluster weight, not that one cluster absorbs all mass with probability tending to one. The richer random cascade description is classical \citep{bovier2004grem} and is not needed for these statements.

\section{Count correction and optimal temperature: complete proofs}
\label{app:optimal}
The count-corrected log odds are exactly
\begin{equation}
 \log\frac{1-\ac}{\ac}=\log\Z_d-\log m_d-\beta X_*.
 \label{eq:countodds}
\end{equation}
Their normalized limit is $F-\alpha_2-\beta s$, giving the fixed-temperature threshold in \eqref{eq:count}.

\subsection{The unique limiting optimum}
\begin{lemma}[Shape of the corrected threshold]\label{lem:shape}
For $0<\rho<1$, $\hc(\beta):=(F(\beta)-\alpha_2)/\beta$ is strictly decreasing on $(0,\sqrt{2\alpha_1})$ and strictly increasing on $(\sqrt{2\alpha_1},\infty)$.
\end{lemma}
\begin{proof}
Set
\begin{equation*}
 c:=\sqrt{2\alpha_1},
 \qquad
 b:=\sqrt{2\alpha},
 \qquad
 b_1:=\sqrt{\frac{2\alpha_1}{\rho}},
 \qquad
 b_2:=\sqrt{\frac{2\alpha_2}{1-\rho}}.
\end{equation*}

Suppose first that $\rho\le\alpha_1/\alpha$. Then $F(\beta)=f(\beta;\alpha,1)$ by \eqref{eq:free}, and \eqref{eq:rem} gives
\begin{equation}
 \hc(\beta)
 =
 \begin{cases}
 \dfrac{\alpha_1}{\beta}+\dfrac{\beta}{2},
 &0<\beta\le b,\\[8pt]
 b-\dfrac{\alpha_2}{\beta},
 &\beta\ge b.
 \end{cases}
 \label{eq:hcfirstcase}
\end{equation}
The derivatives of the two expressions in \eqref{eq:hcfirstcase} are
\begin{equation}
 \hc'(\beta)
 =
 \begin{cases}
 -\dfrac{\alpha_1}{\beta^2}+\dfrac12,
 &0<\beta<b,\\[8pt]
 \dfrac{\alpha_2}{\beta^2},
 &\beta>b.
 \end{cases}
 \label{eq:hcfirstderivative}
\end{equation}
For the first derivative,
\begin{equation*}
 -\frac{\alpha_1}{\beta^2}+\frac12
 \begin{cases}
 <0,&0<\beta<c,\\
 =0,&\beta=c,\\
 >0,&\beta>c.
 \end{cases}
\end{equation*}
Moreover, $c=\sqrt{2\alpha_1}<\sqrt{2(\alpha_1+\alpha_2)}=b$, and hence, \eqref{eq:hcfirstderivative} is negative on $(0,c)$ and positive on $(c,\infty)$. The two formulas in \eqref{eq:hcfirstcase} agree at $\beta=b$, since
\begin{equation*}
 \frac{\alpha_1}{b}+\frac b2
 =
 b-\frac{\alpha_2}{b},
\end{equation*}
which follows from $b^2=2(\alpha_1+\alpha_2)$.

Now suppose that $\rho>\alpha_1/\alpha$. This inequality is equivalent to $b_1<b_2$. Substituting the three formulas for $F$ from \eqref{eq:split} into $\hc$ gives
\begin{equation}
 \hc(\beta)
 =
 \begin{cases}
 \dfrac{\alpha_1}{\beta}+\dfrac{\beta}{2},
 &0<\beta\le b_1,\\[8pt]
 \sqrt{2\rho\alpha_1}+\dfrac{\beta(1-\rho)}{2},
 &b_1\le\beta\le b_2,\\[8pt]
 \sqrt{2\rho\alpha_1}
 +\sqrt{2(1-\rho)\alpha_2}
 -\dfrac{\alpha_2}{\beta},
 &\beta\ge b_2.
 \end{cases}
 \label{eq:hcsecondcase}
\end{equation}
Therefore
\begin{equation}
 \hc'(\beta)
 =
 \begin{cases}
 -\dfrac{\alpha_1}{\beta^2}+\dfrac12,
 &0<\beta<b_1,\\[8pt]
 \dfrac{1-\rho}{2},
 &b_1<\beta<b_2,\\[8pt]
 \dfrac{\alpha_2}{\beta^2},
 &\beta>b_2.
 \end{cases}
 \label{eq:hcsecondderivative}
\end{equation}
Because $0<\rho<1$,
\begin{equation*}
 c=\sqrt{2\alpha_1}
 <\sqrt{\frac{2\alpha_1}{\rho}}
 =b_1.
\end{equation*}
Thus the first expression in \eqref{eq:hcsecondderivative} is negative on $(0,c)$ and positive on $(c,b_1)$, while the other two expressions are strictly positive. The formulas in \eqref{eq:hcsecondcase} agree at $\beta=b_1$ and $\beta=b_2$; explicitly,
\begin{align*}
 \frac{\alpha_1}{b_1}+\frac{b_1}{2}
 &=\sqrt{2\rho\alpha_1}+\frac{b_1(1-\rho)}2,\\
 \sqrt{2\rho\alpha_1}+\frac{b_2(1-\rho)}2
 &=\sqrt{2\rho\alpha_1}
   +\sqrt{2(1-\rho)\alpha_2}
   -\frac{\alpha_2}{b_2},
\end{align*}
using $b_1^2=2\alpha_1/\rho$ and
$b_2^2=2\alpha_2/(1-\rho)$.

In both parameter cases, $\hc'(\beta)<0$ for $0<\beta<c$ and $\hc'(\beta)>0$ for $\beta>c$. Hence the unique minimizer is
\begin{equation*}
 \beta=c=\sqrt{2\alpha_1}.
\end{equation*}
Finally,
\begin{equation*}
 \hc(c)
 =\frac{\alpha_1}{c}+\frac c2
 =c,
\end{equation*}
because $c^2=2\alpha_1$. This proves \eqref{eq:optimal}.
\end{proof}

\subsection{Uniform control over all temperatures}
\begin{lemma}[Maximum scores and compact uniform convergence]\label{lem:uniform}
Let $V_d=d^{-1}\max_{a,j}X_{aj}$ and $B_d=(N_dd)^{-1}\sum_{a,j}X_{aj}$. Then $B_d\pto0$, and $V_d\pto v$ where
\begin{equation}
 v=
 \begin{cases}
 \sqrt{2\alpha},
 &\rho\le\alpha_1/\alpha,\\[2pt]
 \sqrt{2\rho\alpha_1}+\sqrt{2(1-\rho)\alpha_2},
 &\rho>\alpha_1/\alpha.
 \end{cases}
 \label{eq:vmax}
\end{equation}
For each $0<t<T<\infty$, $\sup_{\beta\in[t,T]}|L_d(\beta)-F(\beta)|\pto0$.
\end{lemma}
\begin{proof}
Using \eqref{eq:model} and $N_d=M_dm_d$, write
\begin{align}
 B_d
 &=\frac1{N_dd}\sum_{a=1}^{M_d}\sum_{j=1}^{m_d}
 \left(\sqrt{\rho d}\,G_a+\sqrt{(1-\rho)d}\,H_{aj}\right) \notag\\
 &=\frac{\sqrt{\rho d}}{M_dd}\sum_{a=1}^{M_d}G_a
 +\frac{\sqrt{(1-\rho)d}}{N_dd}
 \sum_{a=1}^{M_d}\sum_{j=1}^{m_d}H_{aj}.
 \label{eq:Bddecomposition}
\end{align}
The two sums in \eqref{eq:Bddecomposition} are independent and centered. Therefore $\E B_d=0$ and
\begin{equation*}
 \Var(B_d)=\frac{\rho d}{M_d^2d^2}M_d
 +\frac{(1-\rho)d}{N_d^2d^2}N_d=\frac{\rho}{M_dd}
 +\frac{1-\rho}{N_dd}.
\end{equation*}
Since $M_d\to\infty$ and $N_d\to\infty$, Chebyshev's inequality gives, for every $\varepsilon>0$,
\begin{equation*}
 \Prob(|B_d|>\varepsilon)
 \le
 \frac1{\varepsilon^2}
 \left(
 \frac{\rho}{M_dd}
 +\frac{1-\rho}{N_dd}
 \right)
 \longrightarrow0.
\end{equation*}
Hence $B_d\pto0$.

We next prove the limit of $V_d$. For every fixed $\beta>0$, $e^{\beta dV_d}\le \Z_d(\beta)\le N_de^{\beta dV_d}$. Taking logarithms and dividing by $\beta d$ gives
\begin{equation}
 \frac{L_d(\beta)}{\beta}
 -\frac{\log N_d}{\beta d}
 \le V_d
 \le\frac{L_d(\beta)}{\beta}.
 \label{eq:Vdsqueeze}
\end{equation}

For each fixed $\beta$, Theorem~\ref{thm:free} gives
\begin{equation}
 \frac{L_d(\beta)}{\beta}
 \pto\frac{F(\beta)}{\beta},
 \qquad
 \frac{\log N_d}{\beta d}
 \longrightarrow\frac{\alpha}{\beta}.
 \label{eq:fixedbetalimits}
\end{equation}
Moreover, combining \eqref{eq:rem} and \eqref{eq:free}, and dividing by $\beta$ gives
\begin{equation*}
 \lim_{\beta\to\infty}\frac{F(\beta)}{\beta}
 =
 \begin{cases}
 \sqrt{2\alpha},
 &\rho\le\alpha_1/\alpha,\\[2pt]
 \sqrt{2\rho\alpha_1}+\sqrt{2(1-\rho)\alpha_2},
 &\rho>\alpha_1/\alpha,
 \end{cases}
 =v.
\end{equation*}

Fix $\varepsilon>0$. Choose a finite number $\beta_0>0$ sufficiently large that
\begin{equation*}
 \left|\frac{F(\beta_0)}{\beta_0}-v\right|
 <\varepsilon,
 \qquad
 \frac{\alpha+1}{\beta_0}<\varepsilon.
\end{equation*}
Since $d^{-1}\log N_d\to\alpha$, for all sufficiently large $d$, we have $d^{-1}\log N_d\le\alpha + 1$. On the event $\left|L_d(\beta_0)-F(\beta_0)\right|
 \le\beta_0\varepsilon$, the upper bound in \eqref{eq:Vdsqueeze} gives
\begin{equation*}
 V_d
 \le\frac{F(\beta_0)}{\beta_0}+\varepsilon
 \le v+2\varepsilon,
\end{equation*}
and the lower bound gives
\begin{equation*}
 V_d
 \ge\frac{F(\beta_0)}{\beta_0}
 -\varepsilon
 -\frac{\alpha+1}{\beta_0}
 \ge v-3\varepsilon.
\end{equation*}
The probability of this event tends to one by
\eqref{eq:fixedbetalimits}. Hence
\begin{equation*}
 \Prob(|V_d-v|>3\varepsilon)\longrightarrow0,
\end{equation*}
and since $\varepsilon>0$ was arbitrary, $V_d\pto v$.

Finally, fix $0<t<T<\infty$. Define
\begin{equation*}
 Q_d:=\frac1d\max_{a,j}|X_{aj}|.
\end{equation*}
Since
\begin{equation*}
 \frac{\partial L_d(\beta)}{\partial\beta}
 =\frac1d\sum_{a,j}p_{aj}(\beta)X_{aj},
\end{equation*}
we have
\begin{equation*}
 \left|\frac{\partial L_d(\beta)}{\partial\beta}\right|
 \le\frac1d\sum_{a,j}p_{aj}(\beta)|X_{aj}|
 \le Q_d.
\end{equation*}
Therefore, for every $\beta,\gamma>0$, $|L_d(\beta)-L_d(\gamma)|\le Q_d|\beta-\gamma|$.

Since every $X_{aj}$ is marginally $N(0,d)$, for every $R>0$, a union bound gives
\begin{equation}
 \Prob(Q_d>R)
 \le2N_d\exp\left(-\frac{dR^2}{2}\right).
 \label{eq:Qdtail}
\end{equation}
Choose $R$ so that $R^2/2>\alpha+1$. Since
$d^{-1}\log N_d\to\alpha$, the right-hand side of
\eqref{eq:Qdtail} tends to zero. Hence $\Prob(Q_d\le R)\longrightarrow1$.

Let $\varepsilon>0$. Because $F$ is continuous on the compact interval $[t,T]$, choose $\delta>0$ such that
\begin{equation*}
 |\beta-\gamma|\le\delta
 \quad\Longrightarrow\quad
 |F(\beta)-F(\gamma)|\le\frac{\varepsilon}{3},
\end{equation*}
and decrease $\delta$ further so that $R\delta\le\varepsilon/3$. Choose finitely many grid points
\begin{equation*}
 t=\beta_0<\beta_1<\cdots<\beta_K=T
\end{equation*}
such that every $\beta\in[t,T]$ is within distance $\delta$ of at least one $\beta_k$. At each fixed grid point, $L_d(\beta_k)\pto F(\beta_k)$. Because the grid contains only finitely many points, the union bound gives
\begin{equation*}
 \Prob\left(
 \max_{0\le k\le K}
 |L_d(\beta_k)-F(\beta_k)|
 >\frac{\varepsilon}{3}
 \right)
 \longrightarrow0.
\end{equation*}

On the event $Q_d\le R$ and $\max_{0\le k\le K}|L_d(\beta_k)-F(\beta_k)|\le\varepsilon/3$, choose for each $\beta\in[t,T]$ a grid point $\beta_k$ satisfying $|\beta-\beta_k|\le\delta$. Then
\begin{align*}
 |L_d(\beta)-F(\beta)|
 &\le
 |L_d(\beta)-L_d(\beta_k)|
 +|L_d(\beta_k)-F(\beta_k)|
 +|F(\beta_k)-F(\beta)| \notag\\
 &\le
 R\delta+\frac{\varepsilon}{3}+\frac{\varepsilon}{3}
 \le\varepsilon.
\end{align*}
Therefore
\begin{equation*}
 \Prob\left(
 \sup_{\beta\in[t,T]}|L_d(\beta)-F(\beta)|
 >\varepsilon
 \right)
 \longrightarrow0.
\end{equation*}
This proves the compact uniform convergence.
\end{proof}

\begin{proof}[Proof of Theorem~\ref{thm:optimal}]
Put $c:=\sqrt{2\alpha_1}$ and
\begin{equation}
 H_d(\beta)=\frac{\log\Z_d(\beta)-\log m_d}{\beta d}.
 \label{eq:Hddefinition}
\end{equation}
Lemma~\ref{lem:shape} gives
\begin{equation}
 h_{\mathrm{cnt}}(\beta)\ge c
 \quad\hbox{for every }\beta>0,
 \qquad
 h_{\mathrm{cnt}}(\beta)=c
 \quad\Longleftrightarrow\quad
 \beta=c.
 \label{eq:uniquecountminimum}
\end{equation}
For every fixed $\beta>0$, \eqref{eq:countodds} and $L_d(\beta)\pto F(\beta)$ give
\begin{equation}
 \frac{1}{d}\log\frac{1-\ac(\beta)}{\ac(\beta)}
 =\beta\left[
 \frac{L_d(\beta)-d^{-1}\log m_d}{\beta}
 -\frac{X_*}{d}\right]
 \pto
 \beta\big[h_{\mathrm{cnt}}(\beta)-s\big].
 \label{eq:fixedcountodds}
\end{equation}
If $s<c$, then the limit in \eqref{eq:fixedcountodds} is strictly positive for every fixed $\beta$, and therefore $\ac(\beta)\pto0$. If $s>c$, choose $\beta=c$.  Then the limit is $c(c-s)<0$, and hence $\ac(c)\pto1$.  This proves the fixed-temperature assertions.

We next prove \eqref{eq:adaptive}.  Jensen's inequality gives
\begin{equation}
 \Z_d(\beta)
 \ge
 N_d\exp\left\{
 \frac{\beta}{N_d}\sum_{a,j}X_{aj}\right\}
 =
 N_d e^{\beta dB_d}.
 \label{eq:JensenZlower}
\end{equation}
where $B_d=(N_dd)^{-1}\sum_{a,j}X_{aj}$. Also, the largest summand in $\Z_d(\beta)$ gives
\begin{equation}
 \Z_d(\beta)\ge e^{\beta dV_d}.
 \label{eq:maxtermZlower}
\end{equation}
Substituting \eqref{eq:JensenZlower} and \eqref{eq:maxtermZlower} into \eqref{eq:Hddefinition} yields, for every $\beta>0$,
\begin{equation}
 H_d(\beta)\ge\frac{\log M_d}{\beta d}+B_d,
 \qquad
 H_d(\beta)\ge V_d-\frac{\log m_d}{\beta d}.
 \label{eq:endpointcontrols}
\end{equation}

The limit $v$ in \eqref{eq:vmax} is strictly larger than $c$. We first aim to prove the lower bound
\begin{equation}
 \inf_{\beta>0}H_d(\beta)\ge c-\varepsilon
 \qquad
 \text{for }
 0<\varepsilon<
 \min\left\{1,\frac{v-c}{4}\right\}.
 \label{eq:adaptiveLowerBound}
\end{equation}
For such an $\varepsilon$, choose fixed numbers $0<t<T<\infty$ such that
\begin{equation}
 \frac{\alpha_1}{2t}>c+1,
 \qquad
 v-\frac{\alpha_2+1}{T}>c+2\varepsilon.
 \label{eq:chooseendpointinterval}
\end{equation}
By Lemma~\ref{lem:uniform} and the definitions of $M_d$ and $m_d$, with probability tending to one, we have $B_d\ge-\varepsilon$, $V_d\ge v-\varepsilon$, $d^{-1}\log M_d\ge\frac{\alpha_1}{2}$ and $d^{-1}\log m_d\le\alpha_2+1$. On this event, \eqref{eq:endpointcontrols} gives
\begin{equation}
  H_d(\beta)\ge\frac{\alpha_1}{2t}-\varepsilon
  \qquad
  \text{for }0<\beta<t,
  \label{eq:smallbetahbound}
\end{equation}
and
\begin{equation}
  H_d(\beta) \ge v-\varepsilon-\frac{\alpha_2+1}{T}
 \qquad
 \text{for }\beta>T.
 \label{eq:largebetahbound}
\end{equation}
Both lower bounds in \eqref{eq:smallbetahbound} and \eqref{eq:largebetahbound} exceed $c-\varepsilon$. On the fixed compact interval $[t,T]$,
\begin{equation}
 \sup_{\beta\in[t,T]}
 \left|H_d(\beta)-h_{\mathrm{cnt}}(\beta)\right|
 \pto0,
 \label{eq:Huniform}
\end{equation}
because $H_d(\beta)=L_d(\beta)/\beta-d^{-1}\log(m_d)/\beta$ and Lemma~\ref{lem:uniform} gives uniform convergence of $L_d$ there. Combining \eqref{eq:Huniform} with \eqref{eq:uniquecountminimum} shows
\begin{equation}
 \inf_{\beta\in[t,T]}H_d(\beta)\ge c-\varepsilon
 \label{eq:middlebetahbound}
\end{equation}
with probability tending to one. Equations \eqref{eq:smallbetahbound}, \eqref{eq:largebetahbound}, and \eqref{eq:middlebetahbound} imply \eqref{eq:adaptiveLowerBound} with probability tending to one. Conversely,
\begin{equation}
 \inf_{\beta>0}H_d(\beta)\le H_d(c)\pto h_{\mathrm{cnt}}(c)=c.
 \label{eq:adaptiveUpperBound}
\end{equation}
Since $\varepsilon>0$ is arbitrary, \eqref{eq:adaptiveLowerBound} and \eqref{eq:adaptiveUpperBound} prove \eqref{eq:adaptive}.

It remains to show uniform failure when $s<c$. Set $\delta:=(c-s)/4>0$ and choose
\begin{equation*}
 \varepsilon:=\min\left\{\frac{\delta}{2},\frac12,\frac{v-c}{8}\right\}.
\end{equation*}
Then $0<\varepsilon<\min\{1,(v-c)/4\}$ and $\varepsilon<\delta$. The lower-bound argument above, with this choice of $\varepsilon$, shows that, with probability tending to one,
\begin{equation}
 H_d(\beta)\ge c-\varepsilon
 \quad\hbox{for every }\beta>0.
 \label{eq:allbetaHlower}
\end{equation}
On the same event, we may also require $X_*/d\le s+\delta$.  For every $\beta\ge t$, \eqref{eq:countodds} and \eqref{eq:allbetaHlower} give
\begin{align}
 \log\frac{1-\ac(\beta)}{\ac(\beta)}
 &=\beta d\left[H_d(\beta)-\frac{X_*}{d}\right]\notag\\
 &\ge\beta d(c-s-\delta-\varepsilon)\notag\\
 &\ge td(c-s-2\delta)
 =\frac{td(c-s)}{2}.
 \label{eq:failuremoderatebeta}
\end{align}
For $0<\beta<t$, use \eqref{eq:JensenZlower} directly in \eqref{eq:countodds}:
\begin{equation}
 d^{-1}\log\frac{1-\ac}{\ac}
 \ge d^{-1}\log M_d+\beta(B_d-X_*/d).
 \label{eq:failuresmallbeta}
\end{equation}
Decrease $t$, if necessary, while preserving \eqref{eq:chooseendpointinterval}, so that
\begin{equation}
 t(|s|+2)<\frac{\alpha_1}{4}.
 \label{eq:choosefailuret}
\end{equation}
With probability tending to one, we have $|B_d|\le1$, $\left|X_*/d\right|\le|s|+1$, and $d^{-1}\log M_d\ge\alpha_1/2$. On this event, \eqref{eq:failuresmallbeta} and \eqref{eq:choosefailuret} imply, for every $0<\beta<t$,
\begin{equation}
 d^{-1}\log\frac{1-\ac(\beta)}{\ac(\beta)}
 \ge\frac{\alpha_1}{2}-t(|s|+2)
 >\frac{\alpha_1}{4}.
 \label{eq:failureSmallBetaFinal}
\end{equation}
Together, \eqref{eq:failuremoderatebeta} and \eqref{eq:failureSmallBetaFinal} show that
\begin{equation*}
 \inf_{\beta>0}\log\frac{1-\ac(\beta)}{\ac(\beta)}
 \xrightarrow{\mathbb P}+\infty.
\end{equation*}
This implies $\sup_{\beta>0}\ac(\beta)\pto0$.
\end{proof}

\begin{proof}[Proof of Proposition~\ref{prop:finite}]
The event $\ac\le1-\varepsilon$ is equivalent to $\Z_d/m_d\ge[\varepsilon/(1-\varepsilon)]e^{\beta sd}$. For every fixed $(a,j)$, $X_{aj}\sim N(0,d)$, and hence $\E e^{\beta X_{aj}}=\exp(\beta^2d/2)$. Using linearity of expectation,
\begin{equation*}
 \E\left[\frac{\Z_d}{m_d}\right]
 =\frac1{m_d}\sum_{a=1}^{M_d}\sum_{j=1}^{m_d}\E e^{\beta X_{aj}}
 =M_d\exp(\beta^2d/2).
\end{equation*}
Markov's inequality proves \eqref{eq:finite}.

Completing the square $\beta^2/2-\beta s=(\beta-s)^2/2-s^2/2$ gives the minimizing positive temperature when $s>0$. This minimizes the upper bound for a given $s$, whereas $\beta=\sqrt{2\alpha_1}$ minimizes the asymptotic signal threshold; these are different optimization problems.
\end{proof}

\section{Averaging, estimation error, and value recovery}
\label{app:averaging}
\begin{proof}[Proof of Theorem~\ref{thm:averaging}]
For every cluster,
\begin{equation*}
 e^{\beta\bar X_a}
 =\exp\left\{\frac1{m_d}\sum_{j=1}^{m_d}\beta X_{aj}\right\}
 \le\frac1{m_d}\sum_{j=1}^{m_d}e^{\beta X_{aj}}.
\end{equation*}
by Jensen's inequality. Summing over $a$ gives $\Z_d^{\mathrm{avg}}=\sum_{a=1}^{M_d}e^{\beta\bar X_a}\le\Z_d/m_d$. Therefore, we have
\begin{equation*}
 \am
 =\frac{e^{\beta X_*}}{e^{\beta X_*}+\Z_d^{\mathrm{avg}}}
 \ge
 \frac{e^{\beta X_*}}{e^{\beta X_*}+\Z_d/m_d}
 =\ac.
\end{equation*}
For the Gaussian model,
\begin{equation*}
 \bar X_a
 =\sqrt{\rho d}\,G_a+R_a,
 \qquad
 R_a:=\sqrt{(1-\rho)d}\,
 \frac1{m_d}\sum_{j=1}^{m_d}H_{aj}.
\end{equation*}
Since $m_d^{-1}\sum_jH_{aj}\sim N(0,m_d^{-1})$, $R_a\sim N\left(0,\frac{(1-\rho)d}{m_d}\right)$. The $R_a$ are independent across clusters and independent of $(G_a)_a$. Hence, by the union bound, for every $u>0$,
\begin{equation*}
 \Prob\left\{\max_a|R_a|>ud\right\}
 \le2M_d\exp\left[-\frac{u^2dm_d}{2(1-\rho)}\right]\longrightarrow0.
\end{equation*}
Put $Y_a:=\sqrt{\rho d}\,G_a$.  For every realized array,
\begin{equation*}
 e^{-\beta\max_a|R_a|}\sum_ae^{\beta Y_a}
 \le\sum_ae^{\beta(Y_a+R_a)}
 \le e^{\beta\max_a|R_a|}\sum_ae^{\beta Y_a}.
\end{equation*}
Taking logarithms and dividing by $d$ gives
\begin{equation*}
 \left|
 \frac1d\log\Z_d^{\mathrm{avg}}
 -\frac1d\log\sum_{a=1}^{M_d}e^{\beta\sqrt{\rho d}G_a}
 \right|
 \le\frac{\beta}{d}\max_a|R_a|
 \pto0.
\end{equation*}
The remaining sum has $M_d=\exp(\alpha_1d+o(d))$ independent $N(0,\rho d)$ scores.  The one-level version of Lemma~\ref{lem:counting} therefore gives
\begin{equation*}
 \frac1d\log\sum_{a=1}^{M_d}e^{\beta\sqrt{\rho d}G_a}
 \pto f(\beta;\alpha_1,\rho),
\end{equation*}
and thus $d^{-1}\log\Z_d^{\mathrm{avg}} \pto f(\beta;\alpha_1,\rho)$. Since
\begin{equation*}
 \log\frac{1-\am}{\am}
 =\log\Z_d^{\mathrm{avg}}-\beta X_*,
\end{equation*}
we have
\begin{equation*}
 \frac{1}{\beta d}\log\frac{1-\am}{\am}
 \pto\frac{f(\beta;\alpha_1,\rho)}{\beta}-s
 =\hm(\beta)-s.
\end{equation*}
Thus $\am\pto1$ when $s>\hm(\beta)$ and $\am\pto0$ when $s<\hm(\beta)$.  Put $b_1:=\sqrt{2\alpha_1/\rho}$. The one-level free energy $f(\beta;\alpha_1,\rho)$ gives
\begin{equation*}
 \hm(\beta)
 =
 \begin{cases}
 \dfrac{\alpha_1}{\beta}+\dfrac{\rho\beta}{2},
 &0<\beta\le b_1,\\[8pt]
 \sqrt{2\rho\alpha_1},
 &\beta\ge b_1.
 \end{cases}
\end{equation*}
The derivative of the first expression is $-\alpha_1/\beta^2+\rho/2<0$ for $0<\beta<b_1$, so it is decreasing and its value at $b_1$ is $\sqrt{2\rho\alpha_1}$. Hence $\inf_{\beta>0}\hm(\beta)=\sqrt{2\rho\alpha_1}$ with equality for every $\beta\ge b_1$. 

Since $0<\rho<1$, $\sqrt{2\rho\alpha_1}<\sqrt{2\alpha_1}$. For every $s$ strictly between these values, $\am\pto1$ for every fixed $\beta\ge b_1$, while Theorem~\ref{thm:optimal} gives $\sup_{\beta>0}\ac(\beta)\pto0$ even when the count-correction temperature is selected after observing the distractor scores.
\end{proof}

\begin{proof}[Proof of Corollary~\ref{cor:singletons}]
Let $Y_1,\ldots,Y_{L_d}$ be the added singleton distractor scores and put $Z_{0,d}(\beta):=\sum_{\ell=1}^{L_d}e^{\beta Y_\ell}$. Since $L_d=\exp(a_0d+o(d))$ and the $Y_\ell$ are independent $N(0,d)$ variables, the one-level free-energy formula gives $d^{-1} \log Z_{0,d}(\beta)\pto f(\beta;a_0,1)$. Count correction leaves each added singleton unchanged because its group size is one. Thus the total distractor weights in the augmented model are
\begin{equation*}
 D_{\mathrm{cnt},0}(\beta)
 :=\frac{\Z_d(\beta)}{m_d}+Z_{0,d}(\beta),
 \qquad
 D_{\mathrm{avg},0}(\beta)
 :=\Z_d^{\mathrm{avg}}(\beta)+Z_{0,d}(\beta).
\end{equation*}
Write
\begin{equation*}
 A^{\mathrm{cnt}}_{*,0}(\beta)
 :=\frac{e^{\beta X_*}}{e^{\beta X_*}+D_{\mathrm{cnt},0}(\beta)},
 \qquad
 A^{\mathrm{avg}}_{*,0}(\beta)
 :=\frac{e^{\beta X_*}}{e^{\beta X_*}+D_{\mathrm{avg},0}(\beta)}.
\end{equation*}
Their exact log odds are
\begin{equation}
 \log\frac{1-A^{\mathrm{cnt}}_{*,0}(\beta)}
 {A^{\mathrm{cnt}}_{*,0}(\beta)}
 =\log D_{\mathrm{cnt},0}(\beta)-\beta X_*,
 \qquad
 \log\frac{1-A^{\mathrm{avg}}_{*,0}(\beta)}
 {A^{\mathrm{avg}}_{*,0}(\beta)}
 =\log D_{\mathrm{avg},0}(\beta)-\beta X_*.
 \label{eq:augmentedlogodds}
\end{equation}
For positive $U,V$, observe that
\begin{equation*}
 \max\{\log U,\log V\}
 \le\log(U+V)
 \le\max\{\log U,\log V\}+\log2.
\end{equation*}
The term $d^{-1}\log2$ vanishes. Applying this bound to the two sums defining $D_{\mathrm{cnt},0}$ and $D_{\mathrm{avg},0}$, and using Theorems~\ref{thm:optimal} and \ref{thm:averaging}, gives
\begin{equation}
 \frac{1}{\beta d}\log
 \frac{1-A^{\mathrm{cnt}}_{*,0}(\beta)}
 {A^{\mathrm{cnt}}_{*,0}(\beta)}
 \pto h_{\mathrm{cnt},0}(\beta)-s,
 \qquad
 \frac{1}{\beta d}\log
 \frac{1-A^{\mathrm{avg}}_{*,0}(\beta)}
 {A^{\mathrm{avg}}_{*,0}(\beta)}
 \pto h_{\mathrm{avg},0}(\beta)-s,
 \label{eq:augmentedthresholdlimits}
\end{equation}
where $h_{\mathrm{cnt},0}(\beta)=\max\{\hc(\beta),f(\beta;a_0,1)/\beta\}$ and $h_{\mathrm{avg},0}(\beta)=\max\{\hm(\beta),f(\beta;a_0,1)/\beta\}$.

Put $c:=\sqrt{2\alpha_1}$. Theorem~\ref{thm:optimal} gives $\hc(\beta)\ge c$ for every $\beta>0$ and $\hc(c)=c$. This implies $h_{\mathrm{cnt},0}(\beta)\ge \hc(\beta)\ge c$ for every $\beta>0$. At $\beta = c$, since $a_0<\alpha_1$, we also have $c>\sqrt{2a_0}$, and the definition \eqref{eq:rem} gives
\begin{equation*}
 \frac{f(c;a_0,1)}{c}=\sqrt{2a_0}<c.
\end{equation*}
Therefore, $h_{\mathrm{cnt},0}(\beta)\ge c$ for every $\beta>0$ and $h_{\mathrm{cnt},0}(c)=c$.

If $s>c$, the first limit in \eqref{eq:augmentedthresholdlimits} at $\beta=c$ is negative, so $A^{\mathrm{cnt}}_{*,0}(c)\pto1$. If $s<c$, then samplewise
\begin{equation*}
 A^{\mathrm{cnt}}_{*,0}(\beta)
 =\frac{e^{\beta X_*}}
 {e^{\beta X_*}+\Z_d(\beta)/m_d+Z_{0,d}(\beta)}
 \le
 \frac{e^{\beta X_*}}
 {e^{\beta X_*}+\Z_d(\beta)/m_d}
 =\ac(\beta).
\end{equation*}
Hence
\begin{equation*}
 \sup_{\beta>0}A^{\mathrm{cnt}}_{*,0}(\beta)
 \le\sup_{\beta>0}\ac(\beta)\pto0
\end{equation*}
by Theorem~\ref{thm:optimal}, which proves failure for every data-dependent temperature choice.

For averaging, Theorem~\ref{thm:averaging} gives $\hm(\beta)=\sqrt{2\rho\alpha_1}$ for $\beta\ge b_1$, and \eqref{eq:rem} gives
\begin{equation*}
 \frac{f(\beta;a_0,1)}{\beta}=\sqrt{2a_0}
 \qquad\hbox{for }\beta\ge\sqrt{2a_0}.
\end{equation*}
Thus, for every fixed $\beta\ge\max\{b_1,\sqrt{2a_0}\}$, we have
\begin{equation*}
 h_{\mathrm{avg},0}(\beta)
 =\max\{\sqrt{2\rho\alpha_1},\sqrt{2a_0}\}
 =c_{\mathrm{avg},0}.
\end{equation*}
If $s>c_{\mathrm{avg},0}$, the second limit in \eqref{eq:augmentedthresholdlimits} is negative, so $A^{\mathrm{avg}}_{*,0}(\beta)\pto1$. Moreover,
\begin{equation*}
 \sqrt{2\rho\alpha_1}<\sqrt{2\alpha_1}=c,
 \qquad
 \sqrt{2a_0}<\sqrt{2\alpha_1}=c,
\end{equation*}
because $0<\rho<1$ and $a_0<\alpha_1$. Hence $c_{\mathrm{avg},0}<c$, proving \eqref{eq:singletonfloor} and the stated separation interval. If $a_0\le\rho\alpha_1$, then $\sqrt{2a_0}\le\sqrt{2\rho\alpha_1}$ and therefore $c_{\mathrm{avg},0}=\sqrt{2\rho\alpha_1}$, so the original interval is unchanged.

There are $L_d+1$ singleton groups, and the planted item is placed uniformly among them. A method that sees only the group memberships and group sizes cannot distinguish these singleton groups. If it selects singleton $\ell$ with probability $q_\ell$, then its success probability is
\begin{equation*}
 \sum_{\ell=1}^{L_d+1}q_\ell\frac1{L_d+1}
 =\frac1{L_d+1}.
\end{equation*}
This bound does not apply to methods that also use the scores.
\end{proof}

\begin{proposition}[Deterministic relative-log-error bound]\label{prop:robust}
Let $D$ and $\widehat D$ be the log odds against the planted item obtained from count correction using the true counts $m_i$ and estimated counts $\widehat m_i$. Under \eqref{eq:relativeerror}, $|\widehat D-D|\le r_d$. If every score used by averaging, including the planted score, is perturbed by at most $u_d$, then its log odds change by at most $2\beta u_d$.
\end{proposition}
\begin{proof}
For count correction,
\begin{equation*}
 e^D=\sum_{i\ne *}e^{\beta(X_i-X_*)}\frac{m_*}{m_i}.
\end{equation*}
Put $e_i:=\log(\widehat m_i/m_i)$. Replacing $m_i$ by $\widehat m_i$ multiplies the $i$th summand by $e^{e_*-e_i}$, which lies in $[e^{-r_d},e^{r_d}]$ by \eqref{eq:relativeerror}. Hence $e^{-r_d}e^D\le e^{\widehat D}\le e^{r_d}e^D$, and taking logarithms gives $|\widehat D-D|\le r_d$.

For averaging, perturbing both scores in $X_i-X_*$ changes this difference by at most $2u_d$. Each odds summand is therefore multiplied by a factor in $[e^{-2\beta u_d},e^{2\beta u_d}]$, giving the second bound in the same way. Thus either conclusion is unchanged at fixed $\beta$ and away from the limiting boundary when the corresponding error is $o(d)$.
\end{proof}

\begin{proposition}[Value error]\label{prop:values}
Let $\widehat v:=\sum_iA_iv_i$, where $A_i\ge0$, $\sum_iA_i=1$, and $\|v_i\|\le R$. Then
\begin{equation*}
 \|\widehat v-v_*\|\le2R(1-A_*).
\end{equation*}
The same bound holds after averaging values within groups when the planted item is a singleton. Consequently, $A_*\pto1$ implies $\widehat v\pto v_*$ and $\E\|\widehat v-v_*\|^2\to0$.
\end{proposition}
\begin{proof}
Since $\sum_iA_i=1$, we have $\widehat v-v_*=\sum_{i\ne *}A_i(v_i-v_*)$. Therefore
\begin{equation*}
 \|\widehat v-v_*\|
 \le\sum_{i\ne *}A_i\|v_i-v_*\|
 \le2R\sum_{i\ne *}A_i
 =2R(1-A_*).
\end{equation*}
An average of vectors with norm at most $R$ also has norm at most $R$, and averaging leaves a singleton planted value unchanged, so the same calculation applies after group averaging. If $A_*\pto1$, the displayed bound gives $\widehat v\pto v_*$. Moreover, $\|\widehat v-v_*\|^2\le4R^2$, so convergence in probability and boundedness imply $\E\|\widehat v-v_*\|^2\to0$.
\end{proof}

\section{A finite-array entropy characterization}
\label{app:entropy}
Let $\widehat\mu=N^{-1}\sum_i\delta_{x_i}$ be the empirical score law. In the tail-probability convention, define the classical entropic value-at-risk \citep{ahmadijavid2012evar} by
\begin{equation*}
 \operatorname{EVaR}_q(\widehat\mu)
 =\inf_{\beta>0}\frac{\log\int e^{\beta x}\,d\widehat\mu(x)-\log q}{\beta}.
\end{equation*}
Our optimized threshold is exactly $d^{-1}\operatorname{EVaR}_{1/M}(\widehat\mu)$. Its entropy dual below is the standard relative-entropy dual: $D(p\|\mathrm{Unif}_N)\le\log M$ is equivalent to $S(p)\ge\log m$. We give a direct finite-array proof for clarity and do not claim this identity as new. The empirical limit in Theorem~\ref{thm:optimal} concerns dependent triangular arrays with an exponentially vanishing tail probability; replacing their empirical law by a Gaussian population law requires that theorem, not just the identity.

Fix integers $M,m>1$, put $N=Mm$, and fix distractor scores $x_1,\ldots,x_N$. Define
\begin{equation*}
 Z(\beta):=\sum_{i=1}^Ne^{\beta x_i},
 \qquad
 p_i(\beta):=\frac{e^{\beta x_i}}{Z(\beta)},
 \qquad
 S(p):=-\sum_{i=1}^Np_i\log p_i,
\end{equation*}
and
\begin{equation*}
 H(\beta):=\frac{\log Z(\beta)-\log m}{\beta d}.
\end{equation*}

\begin{proposition}[EVaR identity and exact temperature calibration]\label{prop:entropy}
Assume that the largest score is unique. Then $H$ has a unique minimizer $\widehat\beta>0$, characterized by $S(p(\widehat\beta))=\log m.$ Moreover, if $\Delta_N$ is the probability simplex, then
\begin{equation}
 \min_{\beta>0}H(\beta)
 =\frac1d\operatorname{EVaR}_{1/M}(\widehat\mu)
 =\frac1d\max_{\substack{q\in\Delta_N\\S(q)\ge\log m}}\sum_iq_ix_i.
 \label{eq:entropydual}
\end{equation}
For a planted score $x_*$, there exists $\beta>0$ such that $A^{\mathrm{cnt}}_*>1/2$ if and only if $x_*/d>\min_{\beta>0}H(\beta)$.
\end{proposition}
\begin{proof}
Differentiation gives
\begin{equation*}
 H'(\beta)=\frac{\log m-S(p(\beta))}{\beta^2d},\qquad
 \frac{d}{d\beta}S(p(\beta))=-\beta\Var_{p(\beta)}(x)<0.
\end{equation*}
The variance is positive because the scores are not all equal and every $p_i(\beta)$ is positive. Moreover, $S(p(\beta))\to\log N$ as $\beta\downarrow0$ and $S(p(\beta))\to0$ as $\beta\to\infty$, where the second limit uses uniqueness of the largest score. Since $0<\log m<\log N$, the equation $S(p(\beta))=\log m$ has one solution. The sign of $H'$ is negative before this solution and positive after it, proving the unique-minimizer claim.

For any $q\in\Delta_N$, nonnegativity of $D(q\|p(\beta))=\sum_i q_i\log (q_i/p_i(\beta))=-S(q)-\beta\sum_i q_ix_i + \log Z(\beta)$ gives $\log Z(\beta)\ge\beta\sum_iq_ix_i+S(q)$ with equality when $q=p(\beta)$. If $S(q)\ge\log m$, then $H(\beta)\ge d^{-1}\sum_iq_ix_i$. At $\beta=\widehat\beta$, the choice $q=p(\widehat\beta)$ satisfies $S(q)=\log m$ and gives equality, proving the last equality in \eqref{eq:entropydual}. For the EVaR equality, use $N=Mm$:
\begin{equation*}
 \log\int e^{\beta x}\,d\widehat\mu(x)-\log(1/M)
 =\log Z(\beta)-\log N+\log M
 =\log Z(\beta)-\log m.
\end{equation*}
Finally,
\begin{equation*}
 A^{\mathrm{cnt}}_*> \frac12
 \quad\Longleftrightarrow\quad
 e^{\beta x_*}>\frac{Z(\beta)}m
 \quad\Longleftrightarrow\quad
 \frac{x_*}{d}>H(\beta).
\end{equation*}
Such a temperature exists exactly when $x_*/d>\min_{\beta>0}H(\beta)$.
\end{proof}
This proposition does not require Gaussian scores, independence, or any particular correlation structure. It describes the best temperature after the distractor scores have been identified and observed, so it is not directly usable when the target location is unknown. In our Gaussian model the largest distractor score is unique almost surely, and the convergence and endpoint bounds proved in Theorem~\ref{thm:optimal} give $\widehat\beta\pto\sqrt{2\alpha_1}$.

\section{The same separation for any fixed finite hierarchy}
\label{app:hierarchy}
Fix $K\ge2$, rates $a_1,\ldots,a_K>0$, and variances $v_1,\ldots,v_K>0$ with $\sum_{\ell=1}^Kv_\ell=1$. At level $\ell$, every node has $n_{\ell,d}:=\lceil e^{a_\ell d}\rceil$ children, and every level-$\ell$ edge has an independent $N(0,v_\ell d)$ score. For a leaf $u$, let $X_u$ be the sum of the $K$ edge scores on its path. Leaves are grouped by their first edge, so there are $n_{1,d}$ groups and each group contains $m_d:=\prod_{\ell=2}^Kn_{\ell,d}$ leaves. 

Write $\mathcal L_d$ for the set of leaves and $\mathcal G_a$ for the leaves in first-level group $a$, and define
\begin{equation*}
 \Z_d^{(K)}(\beta):=\sum_{u\in\mathcal L_d}e^{\beta X_u},
 \qquad
 H_d^{(K)}(\beta):=
 \frac{\log\Z_d^{(K)}(\beta)-\log m_d}{\beta d}.
\end{equation*}
For a singleton planted score $X_*$, the count-corrected mass is
\begin{equation*}
 A_{*,K}^{\mathrm{cnt}}(\beta):=
 \frac{e^{\beta X_*}}
 {e^{\beta X_*}+\Z_d^{(K)}(\beta)/m_d}.
\end{equation*}
For averaging, put
\begin{equation*}
 \bar X_a:=\frac1{m_d}\sum_{u\in\mathcal G_a}X_u,
 \qquad
 \Z_d^{(K),\mathrm{avg}}(\beta):=
 \sum_{a=1}^{n_{1,d}}e^{\beta\bar X_a}.
\end{equation*}
Finally, write
\begin{equation*}
 a_{\le j}:=\sum_{\ell=1}^ja_\ell,\qquad
 V_j:=\sum_{\ell=1}^jv_\ell,\qquad
 a_{\mathrm{tot}}:=a_{\le K}=\sum_{\ell=1}^Ka_\ell.
\end{equation*}

\begin{theorem}[Count correction and averaging in a finite hierarchy]\label{thm:hierarchy}
For every fixed $\beta>0$,
\begin{equation}
 \frac1d\log\Z_d^{(K)}(\beta)\pto F_K(\beta):=
 \max_{\substack{x\in\R^K\\
 \sum_{\ell=1}^jx_\ell^2/(2v_\ell)\le a_{\le j}\ (1\le j\le K)}}
 \left\{
 a_{\mathrm{tot}}-\sum_{\ell=1}^K\frac{x_\ell^2}{2v_\ell}
 +\beta\sum_{\ell=1}^Kx_\ell
 \right\}.
 \label{eq:kvariational}
\end{equation}
The convergence also holds in $L^1$. The count-corrected limiting threshold
\begin{equation*}
 h_K^{\mathrm{cnt}}(\beta):=
 \frac{F_K(\beta)-(a_{\mathrm{tot}}-a_1)}{\beta}
\end{equation*}
satisfies $h_K^{\mathrm{cnt}}(\beta)\ge\sqrt{2a_1}$ and $h_K^{\mathrm{cnt}}(\beta)=\sqrt{2a_1}\ \Longleftrightarrow\ \beta=\sqrt{2a_1}$. Moreover, $\inf_{\beta>0}H_d^{(K)}(\beta)\pto\sqrt{2a_1}$, and if $X_*/d\pto s<\sqrt{2a_1}$, then $\sup_{\beta>0}A_{*,K}^{\mathrm{cnt}}(\beta)\pto0$.

For averaging,
\begin{equation*}
 \frac1d\log\Z_d^{(K),\mathrm{avg}}(\beta)
 \pto f(\beta;a_1,v_1),
\end{equation*}
so its limiting threshold has minimum $\sqrt{2v_1a_1}$.
\end{theorem}
\begin{proof}
We first prove \eqref{eq:kvariational}. This is the $K$-level extension of the bin-counting argument in Appendix~\ref{app:free}. Partition a bounded interval into finitely many bins. For a choice of bins $B_1,\ldots,B_j$ along the first $j$ levels, put $r_\ell:=\inf_{x\in B_\ell}\frac{x^2}{2v_\ell}$. The expected number of depth-$j$ paths whose normalized edge scores lie in these bins is
\begin{equation*}
 \exp\left\{
 d\left(a_{\le j}-\sum_{\ell=1}^jr_\ell+o(1)\right)
 \right\}.
\end{equation*}
If $\sum_{\ell=1}^jr_\ell>a_{\le j}$, Markov's inequality shows that such a prefix is absent with probability tending to one. On the other hand, if $\sum_{\ell=1}^jr_\ell<a_{\le j}$ for $1\le j\le K$, conditional on the selected paths at level $j-1$, the number of level-$j$ extensions is binomial with an exponentially growing mean of rate $a_{\le j}-\sum_{\ell=1}^jr_\ell$. Applying a Chernoff bound successively for $j=1,\ldots,K$ gives the leaf-count exponent $a_{\mathrm{tot}}-\sum_{\ell=1}^Kr_\ell$. Each such leaf contributes exponential weight $\exp\{\beta d\sum_{\ell=1}^Kx_\ell+o(d)\}$. Summing over the finitely many bin choices and then shrinking the bins gives \eqref{eq:kvariational}. Boundary points follow from replacing $x$ by $tx$ and sending $t\uparrow1$. Gaussian union bounds restrict all edge scores to a fixed bounded interval after division by $d$, and the second-moment argument in Lemma~\ref{lem:counting} gives $L^1$ convergence.

Ignoring the constraints in \eqref{eq:kvariational}, differentiation gives $x_\ell=\beta v_\ell$. For this point, $\sum_{\ell=1}^j\frac{x_\ell^2}{2v_\ell}=\frac{\beta^2V_j}{2}$, so it satisfies every constraint whenever $0<\beta\le b_0:=\min_{1\le j\le K}\sqrt{2a_{\le j}/V_j}$. Because $K\ge2$ and every $v_\ell>0$, we have $v_1<1$ and hence $a_{\le1}/V_1=a_1/v_1>a_1$. For $j\ge2$, $a_{\le j}>a_1$ and $V_j\le1$, so again $a_{\le j}/V_j>a_1$. Therefore $b_0>c:=\sqrt{2a_1}$.

For $\beta$ in a neighborhood of $c$, the point $x_\ell=\beta v_\ell$ is feasible and $F_K(\beta)=a_{\mathrm{tot}}+\beta^2/2$. The variational formula is a maximum of affine functions of $\beta$, so $F_K$ is convex. At $\beta=c$, we have $F_K(c)=a_{\mathrm{tot}}+a_1$ with $F_K'(c)=c$. Convexity therefore gives, for every $\beta>0$,
\begin{equation*}
 F_K(\beta)\ge F_K(c)+F_K'(c)(\beta-c)
 =a_{\mathrm{tot}}-a_1+c\beta.
\end{equation*}
After subtracting $a_{\mathrm{tot}}-a_1$ and dividing by $\beta$, $h_K^{\mathrm{cnt}}(\beta)\ge c$. At $\beta=c$ equality holds. If equality also held at some $\beta\ne c$, convexity would force $F_K$ to equal the displayed tangent line on the segment joining $\beta$ and $c$, contradicting $F_K(t)=a_{\mathrm{tot}}+t^2/2$ near $c$. Thus $c$ is the unique minimizer.

It remains to control the empirical infimum over all $\beta$. Let
\begin{equation*}
 \overline X_d:=\frac1{|\mathcal L_d|}\sum_{u\in\mathcal L_d}X_u,
 \qquad
 X_{\max,d}:=\max_{u\in\mathcal L_d}X_u.
\end{equation*}
Each level-$\ell$ edge appears with the same coefficient in $\overline X_d$, and independence gives
\begin{equation*}
 \Var\left(\frac{\overline X_d}{d}\right)
 =\sum_{\ell=1}^K
 \frac{v_\ell}{d\prod_{j=1}^{\ell}n_{j,d}}
 \longrightarrow0.
\end{equation*}
Hence $\overline X_d/d\pto0$. Also, since $\frac1{\beta d}\log\Z_d^{(K)}(\beta) \in \left[\frac{X_{\max,d}}d, \frac{X_{\max,d}}d +\frac{\log|\mathcal L_d|}{\beta d}\right]$, using \eqref{eq:kvariational} and then sending $\beta\to\infty$ shows
\begin{equation*}
 \frac{X_{\max,d}}d\pto
 s_K:=\lim_{\beta\to\infty}\frac{F_K(\beta)}{\beta}.
\end{equation*}
This limit is strictly larger than $c$: choose $c<\beta_0<b_0$. Since $F_K(t)=a_{\mathrm{tot}}+t^2/2$ near $\beta_0$, we have $F_K'(\beta_0)=\beta_0>c$, and convexity implies $s_K\ge\beta_0$.

Jensen's inequality and the largest-term bound give, for every $\beta>0$,
\begin{equation}
 H_d^{(K)}(\beta)
 \ge
 \frac{\log|\mathcal L_d|-\log m_d}{\beta d}
 +\frac{\overline X_d}{d},
 \qquad
 H_d^{(K)}(\beta)
 \ge
 \frac{X_{\max,d}}d-\frac{\log m_d}{\beta d}.
 \label{eq:hierarchyendpoints}
\end{equation}
The first bound controls all sufficiently small $\beta$, because $|\mathcal L_d|/m_d=n_{1,d}$ and $d^{-1}\log n_{1,d}\to a_1$. The second controls all sufficiently large $\beta$, because $X_{\max,d}/d\pto s_K>c$ and $d^{-1}\log m_d\to a_{\mathrm{tot}}-a_1$. On any fixed interval $[t,T]\subset(0,\infty)$, a Gaussian union bound gives $\max_u|X_u|/d=O_{\mathbb P}(1)$, so the functions $d^{-1}\log\Z_d^{(K)}(\beta)$ have bounded random Lipschitz constants with high probability. Fixed-grid convergence from \eqref{eq:kvariational} then gives uniform convergence on $[t,T]$. Combining this with \eqref{eq:hierarchyendpoints} yields $\inf_{\beta>0}H_d^{(K)}(\beta)\pto c$.

Now suppose $X_*/d\pto s<c$, and put $\delta:=(c-s)/4$. With probability tending to one,
\begin{equation*}
 H_d^{(K)}(\beta)\ge c-\delta
 \quad\hbox{for every }\beta>0,
 \qquad
 \frac{X_*}{d}\le s+\delta.
\end{equation*}
Choose $t>0$ small enough that $t(|s|+2)<a_1/4$. For $\beta\ge t$,
\begin{equation*}
 \log\frac{1-A_{*,K}^{\mathrm{cnt}}(\beta)}
 {A_{*,K}^{\mathrm{cnt}}(\beta)}
 =\beta d\left(H_d^{(K)}(\beta)-\frac{X_*}{d}\right)
 \ge\frac{td(c-s)}2.
\end{equation*}
For $0<\beta<t$, Jensen's inequality gives
\begin{equation*}
 \frac1d\log\frac{1-A_{*,K}^{\mathrm{cnt}}(\beta)}
 {A_{*,K}^{\mathrm{cnt}}(\beta)}
 \ge\frac{\log n_{1,d}}d
 +\beta\left(\frac{\overline X_d}{d}-\frac{X_*}{d}\right)
 \ge\frac{a_1}{4}
\end{equation*}
with probability tending to one. Therefore the log odds diverge uniformly over $\beta>0$, proving $\sup_{\beta>0}A_{*,K}^{\mathrm{cnt}}(\beta)\pto0$.

Finally, fix a first-level group $a$. Its averaged score can be written as
\begin{equation*}
 \bar X_a=U_a+R_a,
\end{equation*}
where $U_a\sim N(0,v_1d)$ is the score on its first edge. At level $\ell\ge2$, the group average contains $\prod_{j=2}^{\ell}n_{j,d}$ independent edge scores with equal coefficients. Hence $R_a\sim N(0,dw_d)$, independently across groups and independently of $(U_a)_a$, where
\begin{equation*}
 w_d=\sum_{\ell=2}^K\frac{v_\ell}{\prod_{j=2}^{\ell}n_{j,d}}
 \le (1-v_1)/n_{2,d}\longrightarrow0.
\end{equation*}
For every $u>0$, the union bound gives
\begin{equation*}
 \Prob\left\{\max_{1\le a\le n_{1,d}}|R_a|>ud\right\}
 \le2n_{1,d}\exp\left(-\frac{du^2}{2w_d}\right)
 \longrightarrow0.
\end{equation*}
The log-sum-exp difference between $(\bar X_a)_a$ and $(U_a)_a$ is at most $\beta\max_a|R_a|$, so
\begin{equation*}
 \frac1d\log\Z_d^{(K),\mathrm{avg}}(\beta)
 -\frac1d\log\sum_{a=1}^{n_{1,d}}e^{\beta U_a}
 \pto0.
\end{equation*}
The remaining sum contains $\exp(a_1d+o(d))$ independent $N(0,v_1d)$ scores, and therefore has limiting free energy $f(\beta;a_1,v_1)$. Dividing by $\beta$ and using \eqref{eq:rem} gives the minimum $\sqrt{2v_1a_1}$.
\end{proof}

\section{Beyond Gaussian scores: a cumulant characterization}
\label{app:cramer}
The optimized separation extends to a class of non-Gaussian hierarchies. Non-Gaussian GREM free energies and level-dependent rate functions are classical \citep{jana2007grem}. This extension uses exponential tilting and large deviations \citep{dembo2010large}; it does not assert a new large-deviation principle or a new EVaR duality.

Keep the $K$-level tree, first-level groups, and notation of Appendix~\ref{app:hierarchy}. Replace the score on each level-$\ell$ edge $e$ by $S_{\ell,e}^{(d)}:=\sum_{r=1}^dY_{\ell,e,r}$. All increments are independent, and for each fixed level $\ell$ they have the same centered, nondegenerate distribution. Assume
\begin{equation*}
 \Lambda_\ell(t):=\log\E e^{tY_{\ell,e,r}}<\infty
 \qquad (t\in\R),
\end{equation*}
and $\Lambda_\ell'(t)\to\pm\infty$ as $t\to\pm\infty$. Define
\begin{equation*}
 I_\ell(x):=\sup_{t\in\R}\{tx-\Lambda_\ell(t)\},
 \qquad
 \Lambda(t):=\sum_{\ell=1}^K\Lambda_\ell(t).
\end{equation*}
Also write
\begin{equation*}
 a_{\le j}:=\sum_{\ell=1}^ja_\ell,\qquad
 a_{\mathrm{tot}}:=a_{\le K}=\sum_{\ell=1}^Ka_\ell.
\end{equation*}
and write
\begin{equation*}
 L_{K,d}(\beta):=\frac1d\log\Z_d^{(K)}(\beta),
 \qquad
 A_{*,K}^{\mathrm{avg}}(\beta):=
 \frac{e^{\beta X_*}}
 {e^{\beta X_*}+\Z_d^{(K),\mathrm{avg}}(\beta)}.
\end{equation*}

\begin{theorem}[Non-Gaussian free energy and count correction]\label{thm:cramer-count}
The equation
\begin{equation}
 \beta_c\Lambda'(\beta_c)-\Lambda(\beta_c)=a_1
 \label{eq:cramercounttemp}
\end{equation}
has a unique positive solution. For every fixed $\beta>0$,
\begin{equation}
 L_{K,d}(\beta)\pto F_K(\beta):=
 \max_{\substack{x\in\R^K\\
 \sum_{\ell=1}^jI_\ell(x_\ell)\le a_{\le j}\ (1\le j\le K)}}
 \left\{
 a_{\mathrm{tot}}-\sum_{\ell=1}^KI_\ell(x_\ell)
 +\beta\sum_{\ell=1}^Kx_\ell
 \right\}.
 \label{eq:cramervariational}
\end{equation}
The convergence also holds in $L^1$. Define
\begin{equation*}
 h_K^{\mathrm{cnt}}(\beta):=
 \frac{F_K(\beta)-(a_{\mathrm{tot}}-a_1)}{\beta},
 \qquad
 c_{\mathrm{cnt}}:=\Lambda'(\beta_c).
\end{equation*}
Then $h_K^{\mathrm{cnt}}(\beta)\ge c_{\mathrm{cnt}}$ and $h_K^{\mathrm{cnt}}(\beta)=c_{\mathrm{cnt}}\ \Longleftrightarrow\ \beta=\beta_c$. Moreover, $\inf_{\beta>0}H_d^{(K)}(\beta)\pto c_{\mathrm{cnt}}$. If $X_*/d\pto s<c_{\mathrm{cnt}}$, then $\sup_{\beta>0}A_{*,K}^{\mathrm{cnt}}(\beta)\pto0$, while if $s>c_{\mathrm{cnt}}$, the fixed choice $\beta=\beta_c$ gives $A_{*,K}^{\mathrm{cnt}}(\beta_c)\pto1$.
\end{theorem}
\begin{proof}
\emph{Rate functions.} For each level $\ell$ and $t\in\R$, define the tilted law
\begin{equation*}
 \frac{d\mathbb P_{\ell,t}}{d\mathbb P}(y)
 =e^{ty-\Lambda_\ell(t)}.
\end{equation*}
Differentiating the moment generating function gives
\begin{equation*}
 \Lambda_\ell'(t)=\mathbb E_{\ell,t}[Y_{\ell,e,r}],
 \qquad
 \Lambda_\ell''(t)=\operatorname{Var}_{\ell,t}(Y_{\ell,e,r})>0.
\end{equation*}
The variance is positive because the tilted law has the same support as the original nonconstant distribution. Since the increments are centered, $\Lambda_\ell(0)=\Lambda_\ell'(0)=0$. Thus $\Lambda_\ell'$ is strictly increasing. Because its range is $\R$, for every $x$ there is a unique $t_x$ with $\Lambda_\ell'(t_x)=x$, and $I_\ell(x)=t_xx-\Lambda_\ell(t_x)$. Also, for every $T>0$,
\begin{equation*}
 I_\ell(x)\ge
 T|x|-\max\{\Lambda_\ell(T),\Lambda_\ell(-T)\},
\end{equation*}
so $I_\ell(x)/|x|\to\infty$.

Put $J_\ell(t):=t\Lambda_\ell'(t)-\Lambda_\ell(t)$. For $t>0$, we have $J_\ell'(t)=t\Lambda_\ell''(t)>0$. Moreover, for $t\ge1$,
\begin{equation*}
 J_\ell(t)-J_\ell(1)
 =\int_1^t u\Lambda_\ell''(u)\,du
 \ge\Lambda_\ell'(t)-\Lambda_\ell'(1)\longrightarrow\infty.
\end{equation*}
Hence $\sum_\ell J_\ell(t)=t\Lambda'(t)-\Lambda(t)$ increases continuously from zero to infinity, proving existence and uniqueness of $\beta_c$.

\emph{Free energy.} Let $\bar Y_d=d^{-1}\sum_{r=1}^dY_{\ell,e,r}$ and fix $x\in\R$. Under the product tilted law with parameter $t_x$, the law of large numbers gives $\bar Y_d\to x$. For $\epsilon>0$ and $B_\epsilon(x)=(x-\epsilon,x+\epsilon)$, change of measure gives
\begin{equation*}
 e^{-d[I_\ell(x)+|t_x|\epsilon]}
 \mathbb P_{\ell,t_x}\{\bar Y_d\in B_\epsilon(x)\}
 \le
 \Prob\{\bar Y_d\in B_\epsilon(x)\}
 \le
 e^{-d[I_\ell(x)-|t_x|\epsilon]}.
\end{equation*}
Therefore the probability that a level-$\ell$ edge has normalized score near $x$ has exponential rate $-I_\ell(x)$.

Now partition a bounded interval into finitely many bins. If the first $j$ normalized edge scores lie near $x_1,\ldots,x_j$, the expected number of such depth-$j$ paths has exponential rate $a_{\le j}-\sum_{\ell=1}^jI_\ell(x_\ell)$. If this number is negative, Markov's inequality shows that no such depth-$j$ partial path exists with probability tending to one. If the corresponding rate is positive at every depth, then conditional on the paths at level $j-1$, the number of level-$j$ extensions is binomial with exponentially growing mean. Applying a Chernoff bound successively over the $K$ levels gives the leaf-count rate $a_{\mathrm{tot}}-\sum_{\ell=1}^KI_\ell(x_\ell)$. Multiplying this count by the leaf weight $\exp\{\beta d\sum_\ell x_\ell+o(d)\}$ and maximizing over feasible $x$ gives \eqref{eq:cramervariational}. Boundary points follow by replacing $x$ with $ux$ and sending $u\uparrow1$, since $I_\ell(ux)\le uI_\ell(x)$.

It remains only to remove the bounded-score restriction. For fixed $t,R>0$, a union bound over all level-$\ell$ edges gives
\begin{equation*}
 \Prob\left\{\max_e|S_{\ell,e}^{(d)}|>dR\right\}
 \le
 2\Big(\prod_{j=1}^{\ell}n_{j,d}\Big)
 \exp\{d[\max(\Lambda_\ell(t),\Lambda_\ell(-t))-tR]\}.
\end{equation*}
Choosing $R$ large makes this tend to zero for every level. The same bound, integrated in $R$, gives a uniform second-moment bound for the largest leaf score divided by $d$, and hence the convergence also holds in $L^1$.

\emph{Count threshold.} Without the path-existence constraints, each coordinate maximizes $\beta x-I_\ell(x)$ at $x_\ell=\Lambda_\ell'(\beta)$, with value $\Lambda_\ell(\beta)$. Its cost is $I_\ell(\Lambda_\ell'(\beta))=\beta\Lambda_\ell'(\beta)-\Lambda_\ell(\beta)$. At $\beta=\beta_c$, the total cost is $\sum_\ell J_\ell(\beta_c)=a_1$. The level-$1$ cost is strictly smaller than $a_1$ because $J_\ell(\beta_c)>0$ for $\ell\ge2$, while for $j\ge2$,
\begin{equation*}
 \sum_{\ell=1}^jJ_\ell(\beta_c)\le a_1<a_{\le j}.
\end{equation*}
Thus all constraints are strict at $\beta_c$ and remain strict nearby. On this neighborhood,
\begin{equation*}
 F_K(\beta)
 =a_{\mathrm{tot}}+\sum_{\ell=1}^K
 \max_x\{\beta x-I_\ell(x)\}
 =a_{\mathrm{tot}}+\Lambda(\beta),
 \qquad
 F_K'(\beta)=\Lambda'(\beta).
\end{equation*}
Since $F_K$ is a maximum of affine functions of $\beta$, it is convex. Its tangent at $\beta_c$ gives
\begin{align*}
 F_K(\beta)\ge F_K(\beta_c)+F_K'(\beta_c)(\beta-\beta_c)=a_{\mathrm{tot}}-a_1+\beta\Lambda'(\beta_c),
\end{align*}
where \eqref{eq:cramercounttemp} gives the second equation. Therefore $h_K^{\mathrm{cnt}}(\beta)\ge\Lambda'(\beta_c)=c_{\mathrm{cnt}}$. Equality holds at $\beta_c$. Equality at another point would make $F_K$ affine on the segment joining that point to $\beta_c$, contradicting $F_K''(\beta)=\Lambda''(\beta)>0$ near $\beta_c$.

\emph{All temperatures.} Let
\begin{equation*}
 \overline X_d:=\frac1{|\mathcal L_d|}\sum_{u\in\mathcal L_d}X_u,
 \qquad
 X_{\max,d}:=\max_{u\in\mathcal L_d}X_u.
\end{equation*}
Each level-$\ell$ edge appears with the same coefficient in $\overline X_d$, so
\begin{equation*}
 \Var\left(\frac{\overline X_d}{d}\right)
 =\frac1d\sum_{\ell=1}^K
 \frac{\Lambda_\ell''(0)}{\prod_{j=1}^{\ell}n_{j,d}}
 \longrightarrow0.
\end{equation*}
Hence $\overline X_d/d\pto0$. Also,
\begin{equation*}
 \frac{X_{\max,d}}d
 \le\frac{L_{K,d}(\beta)}{\beta}
 \le\frac{X_{\max,d}}d
 +\frac{\log|\mathcal L_d|}{\beta d}.
\end{equation*}
Using $L_{K,d}(\beta)\pto F_K(\beta)$ and then sending $\beta\to\infty$ gives
\begin{equation*}
 \frac{X_{\max,d}}d\pto
 v:=\lim_{\beta\to\infty}\frac{F_K(\beta)}{\beta}.
\end{equation*}
Moreover $v>c_{\mathrm{cnt}}$: choose $\beta_0>\beta_c$ in the neighborhood where $F_K=a_{\mathrm{tot}}+\Lambda$. Then
\begin{equation*}
 F_K'(\beta_0)=\Lambda'(\beta_0)>
 \Lambda'(\beta_c)=c_{\mathrm{cnt}},
\end{equation*}
and convexity implies $v\ge F_K'(\beta_0)$.

Jensen's inequality and the largest-term bound give, for every $\beta>0$,
\begin{equation*}
 H_d^{(K)}(\beta)
 \ge\frac{\log n_{1,d}}{\beta d}+\frac{\overline X_d}{d},
 \qquad
 H_d^{(K)}(\beta)
 \ge\frac{X_{\max,d}}d-\frac{\log m_d}{\beta d}.
\end{equation*}
The first inequality keeps the minimum away from $\beta=0$, and the second keeps it away from $\beta=\infty$, because $d^{-1}\log n_{1,d}\to a_1$, $d^{-1}\log m_d\to a_{\mathrm{tot}}-a_1$, and $v>c_{\mathrm{cnt}}$. On every fixed interval $[t,T]\subset(0,\infty)$, the derivative of $L_{K,d}$ is bounded in absolute value by $\max_u|X_u|/d=O_{\mathbb P}(1)$. Fixed-grid convergence therefore gives uniform convergence on $[t,T]$. Combining the three temperature ranges and evaluating at $\beta_c$ yields
\begin{equation*}
 \inf_{\beta>0}H_d^{(K)}(\beta)\pto c_{\mathrm{cnt}}.
\end{equation*}

Finally suppose $X_*/d\pto s<c_{\mathrm{cnt}}$ and put $\delta=(c_{\mathrm{cnt}}-s)/4$. With probability tending to one,
\begin{equation*}
 H_d^{(K)}(\beta)\ge c_{\mathrm{cnt}}-\delta
 \quad\text{for every }\beta>0,
 \qquad
 \frac{X_*}{d}\le s+\delta.
\end{equation*}
Choose $t>0$ small enough that $t(|s|+2)<a_1/2$. For $\beta\ge t$,
\begin{equation*}
 \log\frac{1-A_{*,K}^{\mathrm{cnt}}(\beta)}
 {A_{*,K}^{\mathrm{cnt}}(\beta)}
 =\beta d\left(H_d^{(K)}(\beta)-\frac{X_*}{d}\right)
 \ge\frac{td(c_{\mathrm{cnt}}-s)}2.
\end{equation*}
For $0<\beta<t$, Jensen's inequality gives
\begin{equation*}
 \frac1d\log\frac{1-A_{*,K}^{\mathrm{cnt}}(\beta)}
 {A_{*,K}^{\mathrm{cnt}}(\beta)}
 \ge
 \frac{\log n_{1,d}}d
 +\beta\left(\frac{\overline X_d}{d}-\frac{X_*}{d}\right),
\end{equation*}
whose right-hand side is bounded away from zero with probability tending to one. Thus the log odds diverge uniformly over all $\beta>0$, proving
\begin{equation*}
 \sup_{\beta>0}A_{*,K}^{\mathrm{cnt}}(\beta)\pto0.
\end{equation*}
If $s>c_{\mathrm{cnt}}$, then at $\beta=\beta_c$,
\begin{equation*}
 \frac1{\beta_cd}\log
 \frac{1-A_{*,K}^{\mathrm{cnt}}(\beta_c)}
 {A_{*,K}^{\mathrm{cnt}}(\beta_c)}
 \pto c_{\mathrm{cnt}}-s<0,
\end{equation*}
so $A_{*,K}^{\mathrm{cnt}}(\beta_c)\pto1$.
\end{proof}

\begin{theorem}[Averaging beyond Gaussian scores]\label{thm:cramer}
The equation
\begin{equation}
 \beta_r\Lambda_1'(\beta_r)-\Lambda_1(\beta_r)=a_1
 \label{eq:crameravgtemp}
\end{equation}
has a unique positive solution. Define
\begin{equation*}
 F_{\mathrm{avg}}(\beta):=
 \max_{I_1(x)\le a_1}\{a_1-I_1(x)+\beta x\},
 \qquad
 c_{\mathrm{avg}}:=\Lambda_1'(\beta_r).
\end{equation*}
Then, for every fixed $\beta>0$,
\begin{equation*}
 \frac1d\log\Z_d^{(K),\mathrm{avg}}(\beta)
 \pto F_{\mathrm{avg}}(\beta).
\end{equation*}
The threshold $F_{\mathrm{avg}}(\beta)/\beta$ has minimum $c_{\mathrm{avg}}$ and equals $c_{\mathrm{avg}}$ for every $\beta\ge\beta_r$. Moreover, $c_{\mathrm{avg}}<c_{\mathrm{cnt}}$. Hence for every $s\in(c_{\mathrm{avg}},c_{\mathrm{cnt}})$, averaging retrieves at every fixed $\beta\ge\beta_r$, whereas count correction fails for every temperature choice. The groups must still be known and the planted item must still be a singleton.
\end{theorem}
\begin{proof}
Put $J_1(t):=t\Lambda_1'(t)-\Lambda_1(t)$. Since $J_1'(t)=t\Lambda_1''(t)>0$, $J_1(0)=0$, and $J_1(t)\to\infty$, equation \eqref{eq:crameravgtemp} has a unique positive solution.

For group $a$, write $\bar X_a=S_{1,a}^{(d)}+R_a$, where $R_a$ is the averaged contribution from levels $2,\ldots,K$. Put $q_{\ell,d}=\prod_{j=2}^{\ell}n_{j,d}$ for $\ell\ge2$; then $R_a$ is the sum, over these levels, of averages of $q_{\ell,d}$ independent edge scores. Hence
\begin{equation*}
 \frac1d\log\E e^{tR_a}
 =\sum_{\ell=2}^Kq_{\ell,d}\Lambda_\ell(t/q_{\ell,d})
 \longrightarrow0
\end{equation*}
for every fixed real $t$, since $\Lambda_\ell(z)=\Lambda_\ell''(0)z^2/2+o(z^2)$ at zero. For $u>0$, choose a fixed $t>a_1/u$. A union bound and exponential Markov inequalities at $t$ and $-t$ give
\begin{equation*}
 \Prob\{\max_a|R_a|>du\}\le
 2n_{1,d}\exp\{-dtu+o(d)\}\longrightarrow0.
\end{equation*}
Averaging therefore leaves the one-level root free energy $F_{\mathrm{avg}}$ defined in the theorem.
For positive $\beta$, the optimizer is $\Lambda_1'(\beta)$ until $\beta_r$ and the rightmost feasible point $c_{\mathrm{avg}}=\Lambda_1'(\beta_r)$ thereafter. Thus the threshold decreases to $c_{\mathrm{avg}}$ and equals it for $\beta\ge\beta_r$. In particular
\begin{equation*}
 c_{\mathrm{avg}}=\min_{\beta>0}\frac{a_1+\Lambda_1(\beta)}{\beta}.
\end{equation*}
At the count optimum,
\begin{equation*}
 c_{\mathrm{cnt}}=
 \frac{a_1+\Lambda_1(\beta_c)}{\beta_c}
 +\frac{\sum_{\ell=2}^K\Lambda_\ell(\beta_c)}{\beta_c}
 >c_{\mathrm{avg}},
\end{equation*}
because the second summand is strictly positive. This proves the gap and all retrieval assertions.
\end{proof}

\paragraph{Concrete non-Gaussian example.} Let $Y_\ell=b_\ell(P_\ell^+-P_\ell^-)$ for independent Poisson variables of mean $\lambda_\ell/2$, with $b_\ell,\lambda_\ell>0$. Then
\begin{equation*}
 \Lambda_\ell(t)=\lambda_\ell(\cosh(b_\ell t)-1),\qquad
 I_\ell(x)=\frac{x}{b_\ell}\operatorname{arsinh}\frac{x}{\lambda_\ell b_\ell}
 -\sqrt{(x/b_\ell)^2+\lambda_\ell^2}+\lambda_\ell.
\end{equation*}
These distributions meet every hypothesis and allow direct numerical optimization of \eqref{eq:cramervariational}. For Gaussian increments $\Lambda_\ell(t)=v_\ell t^2/2$ with $\sum_\ell v_\ell=1$, Theorems~\ref{thm:cramer-count} and \ref{thm:cramer} recover exactly $\beta_c=c_{\mathrm{cnt}}=\sqrt{2a_1}$ and $c_{\mathrm{avg}}=\sqrt{2v_1a_1}$. The dependence on the leaf marginal cumulant, rather than its variance alone, matters beyond Gaussian scores. Heavy-tailed increments without exponential moments and bounded-support saturation regimes are not covered by these theorems.

\section{Extensions and boundary cases}
\label{app:extensions}
\subsection{Partial correction and a second optimization boundary}
Fix $0<\gamma\le1$. Partial count correction divides every distractor weight by $m_d^\gamma$, so $\gamma=1$ is full count correction and smaller $\gamma$ leaves a factor $m_d^{1-\gamma}$ of the within-cluster multiplicity. Its distractor denominator and limiting threshold are
\begin{equation*}
 D_{\gamma,d}(\beta):=\frac{\Z_d(\beta)}{m_d^\gamma},
 \qquad
 h_\gamma(\beta):=\frac{F(\beta)-\gamma\alpha_2}{\beta}.
\end{equation*}
Put $c_\gamma:=\alpha_1+(1-\gamma)\alpha_2$, the total distractor growth rate left after the partial correction.
\begin{proposition}[Optimal partial correction]\label{prop:partial}
Assume $0<\rho<1$. If $\rho c_\gamma\le\alpha_1$, then
\begin{equation}
 \beta_\gamma=\sqrt{2c_\gamma},\qquad
 \min_{\beta>0}h_\gamma(\beta)=\sqrt{2c_\gamma}.
 \label{eq:partial1}
\end{equation}
If $\rho c_\gamma>\alpha_1$, then
\begin{equation}
 \beta_\gamma=\sqrt{\frac{2(1-\gamma)\alpha_2}{1-\rho}},\qquad
 \min_{\beta>0}h_\gamma(\beta)
 =\sqrt{2\rho\alpha_1}+\sqrt{2(1-\rho)(1-\gamma)\alpha_2}.
 \label{eq:partial2}
\end{equation}
The formulas agree at equality. The minimizer is unique for $\gamma>0$.
\end{proposition}
\begin{proof}
In the merged regime the first branch is $c_\gamma/\beta+\beta/2$, whose optimum lies strictly before $\sqrt{2\alpha}$. The last branch is $\sqrt{2\alpha}-\gamma\alpha_2/\beta$, strictly increasing. Also $\rho\le\alpha_1/\alpha$ implies $\rho c_\gamma<\alpha_1$. In the split regime, the first branch has the same expression. Its optimum lies at or before $b_1$ precisely when $\rho c_\gamma\le\alpha_1$. The middle branch is
\begin{equation*}
 \sqrt{2\rho\alpha_1}+\frac{(1-\gamma)\alpha_2}{\beta}+\frac{\beta(1-\rho)}2.
\end{equation*}
Under the reverse strict inequality its optimum in \eqref{eq:partial2} lies strictly between $b_1$ and $b_2$. Otherwise its derivative is nonnegative from $b_1$ onward. The last branch $v-\gamma\alpha_2/\beta$ is strictly increasing. The derivative signs and continuity prove the formulas and uniqueness. The case $\gamma=1$ always satisfies \eqref{eq:partial1}.
\end{proof}

\subsection{Unequal multiplicities and a relevant group}
Fix a finite number $T$ of distractor types. Type $t$ has
\begin{equation*}
 M_{t,d}:=\lceil e^{a_{1t}d}\rceil
 \quad\text{clusters},\qquad
 m_{t,d}:=\lceil e^{a_{2t}d}\rceil
 \quad\text{tokens per cluster},
\end{equation*}
correlation $\rho_t\in(0,1)$, and scores $\mu_td+\sigma_tX^{(t)}_{aj}$ with $\sigma_t>0$ where $X^{(t)}$ follows \eqref{eq:model} with rates $(a_{1t},a_{2t})$ and correlation $\rho_t$. Let $F_t$ denote its free energy from Theorem~\ref{thm:free}, and let $\Z_{t,d}$ be its partition function including the shift and scale above. Applying exponent-$\gamma$ count correction gives the total distractor denominator
\begin{equation*}
 D_d^{\mathrm{types}}(\beta)
 :=\sum_{t=1}^T
 \frac{\Z_{t,d}(\beta)}{m_{t,d}^\gamma}.
\end{equation*}
Suppose the relevant group has $r_d$ identical copies of score $X_*$, where $d^{-1}\log r_d\to a_*\ge0$ and $X_*/d\pto s$. Its corrected numerator and total attention mass are
\begin{equation*}
 r_d^{1-\gamma}e^{\beta X_*},
 \qquad
 A_{\mathrm{rel}}
 :=\frac{r_d^{1-\gamma}e^{\beta X_*}}
 {r_d^{1-\gamma}e^{\beta X_*}+D_d^{\mathrm{types}}(\beta)}.
\end{equation*}
\begin{proposition}[Finitely many distractor types]\label{prop:types}
Under the setup above,
\begin{equation*}
 \frac1d\log\frac{1-A_{\mathrm{rel}}}{A_{\mathrm{rel}}}
 \pto\max_{1\le t\le T}\{\beta\mu_t+F_t(\beta\sigma_t)-\gamma a_{2t}\}
 -\beta s-(1-\gamma)a_*.
\end{equation*}
A strictly negative limit implies success and a strictly positive limit implies failure. Cross-type independence is unnecessary if each type has the stated marginal model.
\end{proposition}
\begin{proof}
Each corrected type partition has the asserted exponent by Theorem~\ref{thm:free}. The logarithm of a sum of $T$ positive numbers differs from their maximum log by at most $\log T=o(d)$. A union bound gives simultaneous component convergence. The relevant numerator is $r_d^{1-\gamma}e^{\beta X_*}$. Subtract its logarithm and use the logistic identity.
\end{proof}
Full correction removes the relevant group's multiplicity advantage as well as distractor multiplicities. Individual relevant-token mass is $A_{\mathrm{rel}}/r_d$, so semantic-group retrieval and retrieval of one specific token are different objectives.

\subsection{Exact and approximate replication}
Consider a scalar count rule that multiplies every item in a group of size $m$ by the same factor $w(m)>0$. A group with scores $x_1,\ldots,x_m$ then contributes
\begin{equation*}
 C_w(x_1,\ldots,x_m)
 :=w(m)\sum_{j=1}^me^{\beta x_j}.
\end{equation*}
Call the rule replication invariant if replacing every score by $k$ identical copies leaves this contribution unchanged for every integer $k\ge1$.
\begin{proposition}[Uniqueness within scalar count rules]\label{prop:duplicates}
Replication invariance holds if and only if $w(m)=w(1)/m$. With this rule, replicating an entire score list leaves $m^{-1}\sum_je^{\beta x_j}$ unchanged, whereas replicating only selected nonidentical scores need not preserve it.
\end{proposition}
\begin{proof}
Replicate a singleton $m$ times to get $m w(m)=w(1)$; the converse is immediate. Replication of a whole list multiplies its exponential sum and its size by the same factor. Selective replication changes the empirical distribution of scores.
\end{proof}
For approximate replication, let $\bar x=m^{-1}\sum_jx_j$. If $|x_j-\bar x|\le u$ for all group members, Jensen's inequality and the termwise upper bound give
\begin{equation*}
 0\le\log\left(m^{-1}\sum_je^{\beta x_j}\right)-\beta\bar x\le\beta u.
\end{equation*}
Thus uniformly $o(d)$ score spread makes count correction and averaging equivalent at the exponential scale. With fixed $\rho<1$ and exponentially many children, that uniform spread need not be $o(d)$.

\subsection{No universal effective context length}
An i.i.d. Gaussian context with $\exp(a_{\mathrm{eff}}d+o(d))$ distractors has limiting free energy $f(\beta;a_{\mathrm{eff}},1)$. A temperature-independent effective context length would therefore be a single rate $a_{\mathrm{eff}}$ satisfying
\begin{equation*}
 F(\beta)=f(\beta;a_{\mathrm{eff}},1)
 \qquad\text{for every }\beta>0.
\end{equation*}
The question is whether the correlated two-level model can be replaced by such an i.i.d. context at all temperatures.
\begin{corollary}\label{cor:neff}
If $\rho>\alpha_1/\alpha$, no such rate $a_{\mathrm{eff}}$ exists.
\end{corollary}
\begin{proof}
Sending $\beta\downarrow0$ forces $a_{\mathrm{eff}}=\alpha$. But $b_1<\min(b_2,\sqrt{2\alpha})$, and on the nonempty interval between them,
\begin{equation*}
 f(\beta;\alpha,1)-F(\beta)=\frac\rho2(\beta-b_1)^2>0.
\end{equation*}
This contradicts equality. The strict split condition is essential; in the merged regime the iid free energy holds.
\end{proof}

\subsection{Endpoint and scaling checks}
At $\rho=0$, all scores are i.i.d. and $F=f(\beta;\alpha,1)$. Count correction still has unique optimum $\sqrt{2\alpha_1}$. Averaging has limiting free energy $\alpha_1$, hence threshold $\alpha_1/\beta$ and infimum zero as $\beta\to\infty$. At $\rho=1$,
\begin{equation*}
 \Z_d=m_d\sum_ae^{\beta\sqrt dG_a},\qquad F=\alpha_2+f(\beta;\alpha_1,1).
\end{equation*}
Count correction and averaging coincide exactly. Their minimum is $\sqrt{2\alpha_1}$, attained by every $\beta\ge\sqrt{2\alpha_1}$, so uniqueness in Theorem~\ref{thm:optimal} excludes this endpoint.

At fixed $\beta$, a correction with uniformly $o(d)$ log weight cannot change an off-boundary threshold. Fixed or subexponential multiplicities therefore cannot generate the order-$d$ correction studied here. The two-level model assumes $\alpha_2>0$ rather than a fixed number of duplicates.

For conventional scaled dot-product pre-softmax scores $t_dX/\sqrt d$, a finite-size plug-in choice is
\begin{equation*}
 t_{\mathrm{opt},d}=\sqrt{2\log M_d},\qquad
 \beta_{\mathrm{opt},d}=\sqrt{2\log M_d/d}\longrightarrow\sqrt{2\alpha_1}.
\end{equation*}
Compact uniform convergence justifies this plug-in choice. The scale depends on marginal Gaussian variance and on supplied groups; it is not a universal scaling law for bounded or normalized attention models.

\section{Numerical protocol}
\label{app:numerics}
\begin{figure}[t]
\centering\includegraphics[width=\linewidth]{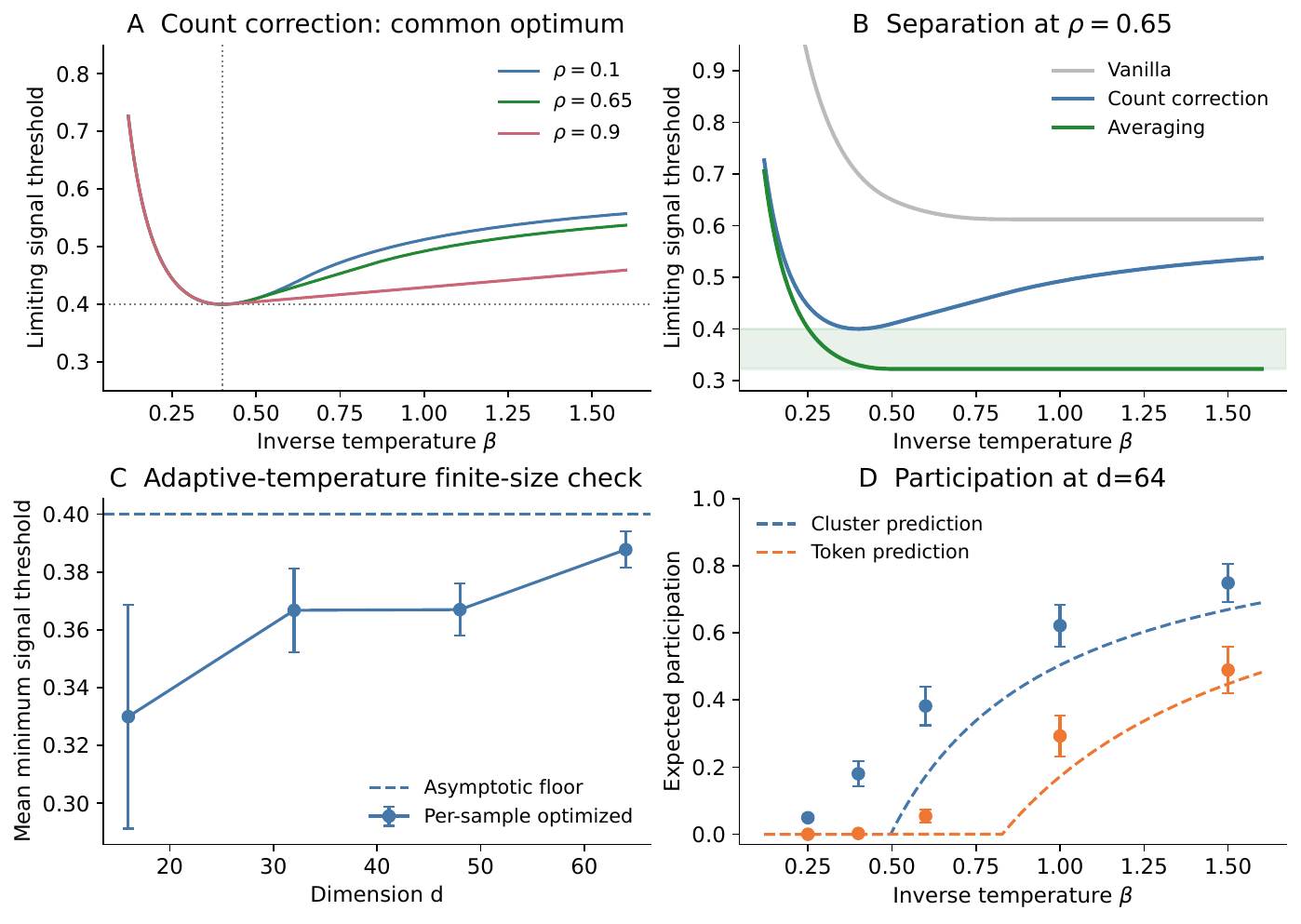}
\caption{Analytical retrieval thresholds and finite-size Gaussian checks. Dashed lines in the simulation panels are infinite-size predictions; markers are empirical means with $1.96$ standard errors. These are synthetic model checks.}\label{fig:checks}
\end{figure}
Figure~\ref{fig:checks} retains the four Gaussian diagnostics. Panel A shows the common count-correction optimum as correlation changes. Panel B compares the required signals for ordinary attention, count correction, and averaging. Panel C optimizes the temperature for each finite array and compares the required signal to the asymptotic limit. Panel D distinguishes cluster participation from token participation.

The scripts use NumPy, SciPy, and Matplotlib on CPU. The theory verification uses seed 20260910 and 120 parameter configurations: $\alpha_1,\alpha_2$ log-uniform on $[e^{-3},e^{0.5}]$, $\rho$ uniform on $[0.015,0.985]$, and $\beta$ log-uniform on $[e^{-3},e^2]$. SLSQP optimizes the variational problem from two starting points, independently of its closed-form case distinction. Scalar optimization checks the count and partial-count optima. Additional checks cover covariance monotonicity, Jensen comparisons, replication, bounded count errors, and the finite-sample interpolation identity. The last check has Monte Carlo sampling error and is not symbolic verification.

The Gaussian simulation uses seed 271828 with $(\alpha_1,\alpha_2,\rho)=(0.08,0.12,0.65)$, $d\in\{16,32,48,64\}$, and 64 independent arrays at each dimension. It evaluates $\beta\in\{0.25,0.4,0.6,1,1.5\}$ on each shared array. The largest array has $168\times2165=363720$ distractor scores. Each per-array optimum is searched on $[0.025,4]$; this bounded search is not the all-temperature proof. Plot error bars are $1.96$ times the sample standard error, interpreted as descriptive normal approximations.

The extension verification uses seed 314159 to compare entropy-root calibration with scalar optimization and a separately constrained 20-dimensional entropy optimization on 20 arrays. It also checks the high-temperature tangent bound by directly optimizing the prefix-constrained variational problem for 40 hierarchies of depths three through five. The maximum entropy-duality discrepancy was about $5.1\times10^{-6}$; numerical tolerances and raw discrepancies are recorded, and solver success is not a proof.

The non-Gaussian verification uses seed 161803 and 30 hierarchies of depths two through five, with independent level-specific Poisson rates and jump sizes. It compares the cumulant optimum with direct prefix-constrained optimization and checks the averaged root floor through a separate rate-function inversion.

Agreement with a Gaussian simulator checks the stipulated model's implementation, not its validity for trained heads. The separate frozen-sentence experiment is described in Appendix~\ref{app:learned}; no training experiment or implemented LUCID comparison is claimed.

\section{Frozen-encoder experiment: protocol and additional results}
\label{app:learned}

\paragraph{Data and representation.} BANKING77 contains 10,003 training and 3,080 test sentences across 77 intents \citep{casanueva2020banking}. We process the official training split first and test split second, discarding repeated texts after lowercasing and collapsing whitespace, including test texts already encountered in training. This removes four training and eight test rows. With seed 20260913, we select 32 training sentences per intent as a key pool, then four validation and 12 evaluation queries per intent from the test split. Validation and evaluation share the intent vocabulary and key pool, but no normalized query text. We retain all selected queries regardless of retrieval outcome. The resulting scores are not an official BANKING77 benchmark result. Dataset labels construct the controlled contexts and evaluate outputs; the retrieval functions receive scores, groups, and values, without the query label.

The encoder is \texttt{sentence-transformers/all-MiniLM-L6-v2}, revision \texttt{1110a243fdf4706b3f48f1d95db1a4f5529b4d41}, using its official float32 ONNX export on CPU. Following the model card \citep{minilm2021model}, we apply attention-mask mean pooling and normalize the resulting 384-dimensional sentence vectors. The maximum input length is 256 tokens. The 3,696 selected sentences are encoded in batches of 32; no parameters are trained. Cosine scores are computed from these vectors. A group score is the arithmetic mean of its member scores, equivalently the dot product with the mean of the normalized key vectors, without normalizing that mean again. BANKING77 is distributed under CC BY 4.0 and the encoder under Apache 2.0. Model pretraining overlap has not been audited.

\paragraph{Paired contexts.} For each seed in $\{11,23,47\}$, we randomly order each intent's 32 keys. For each target intent, we choose 38 of the other 76 intents to form larger distractor groups; the remaining 38 are singleton distractors. Queries of the same intent share these draws. The six conditions use the same ordering: one key per intent; eight copies of the first key in the selected distractor groups; four distinct keys; eight distinct keys; eight distractor keys with four relevant keys; and a mixed relevant group. In the last condition, three keys are moved from the first selected distractor group into the relevant singleton's group. This preserves the token set and leaves five keys in the donor group. Context sizes are respectively 77, 343, 191, 343, 346, and 343 keys. This construction deliberately controls whether relevant keys are averaged; it does not estimate a partition or sample naturally occurring context frequencies.

\paragraph{Methods and values.} With cosine score $c_i$, multiplier $t$, and supplied group size $m_i$, ordinary and corrected token weights are proportional to $e^{tc_i}$ and $e^{tc_i}/m_i$. Averaging assigns weight $e^{t\bar c_g}$ and value $\bar v_g$ to group $g$. Individual values are the 77-dimensional one-hot intent vectors. Thus the output's relevant coordinate sums the attention on relevant values; with several relevant keys it is not a single-token mass. A fourth baseline selects one member uniformly per group and applies softmax with equal group priors. Its selection uses the context seed, not query scores. In the mixed condition it may discard the sole relevant key, yielding zero relevant mass and infinite negative log mass. It therefore has no finite validation objective there; its deterministic smallest-grid fallback is archived but excluded from Table~\ref{tab:learned-detail}.

\paragraph{Temperature and uncertainty.} Each method and condition selects one $t$ from 81 logarithmically spaced values between 0.25 and 1024 by minimizing mean validation negative log relevant value mass over the 308 queries and three seeds. The same chosen value is applied to all evaluation queries and context seeds. All three primary methods choose an interior grid point. The multiplier $t$ acts on cosine scores; we do not identify it with the Gaussian $\beta$ or fit the asymptotic square-root thresholds to these bounded scores. Table~\ref{tab:learned-detail} reports top-1 accuracy of the output intent at the same selected multiplier. Figure~\ref{fig:learned} shows evaluation curves descriptively, without using them to select temperatures. Bootstrap intervals use 2,000 paired resamples of the 77 intent blocks (seed 8831), after averaging the three context seeds within each query. They describe variation across these intents conditional on this encoder, corpus, and key draws; the 2,772 query-context pairs are not treated as independent queries.

\begin{figure}[ht]
\centering\includegraphics[width=\linewidth]{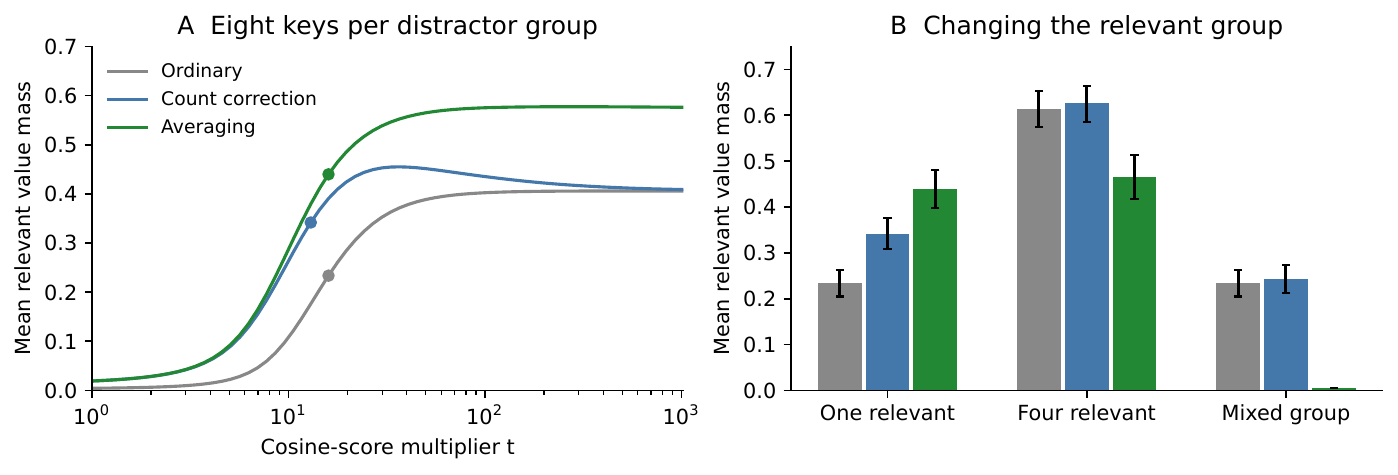}
\caption{Frozen sentence scores. A: evaluation temperature curves for eight-key distractor groups; dots mark independently validation-selected temperatures. B: relevant value mass after changing the relevant group, with 95\% intent-block bootstrap intervals. The reversal for four relevant keys occurs even though all four carry the correct intent label.}\label{fig:learned}
\end{figure}

\paragraph{An all-temperature oracle.} For singleton-target conditions we additionally give count correction the true relevant position and optimize its relevant mass separately for each evaluation query. With distractor scores $c_i$ and target score $c_*$, minimize the convex log-odds function $\log\sum_{i\ne *}\exp[t(c_i-c_*)]/m_i$ over $t\ge0$. Its derivative is monotone, so bracketing and scalar root finding locate an interior minimum; the implementation also handles the limits at zero and infinity. This objective maximizes mass for a known target and differs from minimizing the signal threshold in Proposition~\ref{prop:entropy}. With eight-key distractor groups, oracle count correction has mean mass $0.511$, above calibrated averaging's $0.440$. Only seven of the 2,772 evaluation query-context pairs have averaging mass above $0.9$ while the count oracle cannot reach $0.9$. Moreover, the oracle optimum is at infinity for 40.7\% of these pairs. Finite optimal calibration across queries must therefore not be confused with the theorem's unique finite threshold optimum. The experiment does not establish an asymptotic or uniform all-temperature lower bound on learned scores.

\paragraph{Failures and checks.} With four relevant keys, averaging loses $0.161$ mean relevant mass against count correction, with paired 95\% interval $[-0.181,-0.142]$. High-scoring individual matches can be useful on the relevant side even with correct intent groups. The mixed condition also changes the value: only one quarter of the averaged target-group value is relevant, so its relevant mass cannot exceed $1/4$. The three contaminating keys share one other intent; its output coordinate is always at least three times the target's, making zero top-1 accuracy structural in this stress test. We do not interpret it as a separate empirical discovery. For singleton-target contexts, Jensen ordering holds at every evaluated temperature; exact-copy invariance and equality of all methods on all-singleton contexts hold to $1.1\times10^{-13}$. Independent small-array tests compare mixed values with direct softmax sums and check interior and endpoint oracle solutions against dense temperature search.

\begin{table}[!ht]
\centering\small
\caption{Temperature $t$, mean relevant value mass $V$, and top-1 output-intent accuracy (\%) at that temperature. Temperatures minimize validation negative log relevant value mass, not top-1 error.}
\label{tab:learned-detail}
\resizebox{0.8\linewidth}{!}{%
\begin{tabular}{llrrr}
\toprule
Context & Method & $t$ & $V$ & Accuracy (\%) \\
\midrule
One key per intent & Ordinary & 14.42 & 0.379 & 51.8 \\
One key per intent & Count correction & 14.42 & 0.379 & 51.8 \\
One key per intent & Averaging & 14.42 & 0.379 & 51.8 \\
One key per intent & One representative & 14.42 & 0.379 & 51.8 \\
Exact copies (8) & Ordinary & 17.75 & 0.288 & 36.3 \\
Exact copies (8) & Count correction & 14.42 & 0.379 & 51.8 \\
Exact copies (8) & Averaging & 14.42 & 0.379 & 51.8 \\
Exact copies (8) & One representative & 14.42 & 0.379 & 51.8 \\
Same intent (4) & Ordinary & 14.42 & 0.267 & 39.4 \\
Same intent (4) & Count correction & 13.00 & 0.345 & 53.5 \\
Same intent (4) & Averaging & 14.42 & 0.406 & 56.7 \\
Same intent (4) & One representative & 14.42 & 0.381 & 53.3 \\
Same intent (8) & Ordinary & 16.00 & 0.234 & 33.1 \\
Same intent (8) & Count correction & 13.00 & 0.342 & 53.2 \\
Same intent (8) & Averaging & 16.00 & 0.440 & 57.6 \\
Same intent (8) & One representative & 14.42 & 0.375 & 51.6 \\
Four relevant keys & Ordinary & 21.86 & 0.614 & 72.1 \\
Four relevant keys & Count correction & 21.86 & 0.626 & 72.8 \\
Four relevant keys & Averaging & 19.70 & 0.465 & 59.1 \\
Four relevant keys & One representative & 14.42 & 0.366 & 50.3 \\
Mixed relevant group & Ordinary & 16.00 & 0.234 & 33.1 \\
Mixed relevant group & Count correction & 16.00 & 0.243 & 34.2 \\
Mixed relevant group & Averaging & 4.59 & 0.005 & 0.0 \\
\bottomrule
\end{tabular}%
}
\end{table}

The protocol was written before inspecting retrieval results; the oracle rescue fraction is an additional descriptive analysis. This single-encoder, single-dataset study neither isolates correlation at fixed marginal score distributions nor tests internal attention heads, context-length scaling, or unknown-group recovery.

\section{Frozen-transformer internal-attention diagnostic}
\label{app:transformer}

This diagnostic asks whether count correction and averaging remain mechanically distinct inside an actual pretrained transformer. It is deliberately secondary to the frozen sentence-retrieval study in Section~\ref{sec:experiments}: the intervention is causal, but the model is weak on the controlled task and the downstream results do not support a general advantage for averaging.

\paragraph{Model and task.} We freeze the masked-language model \texttt{distilbert-base-uncased} \citep{sanh2019distilbert} at revision \texttt{12040accade4e8a0f71eabdb258fecc2e7e948be}. Synthetic prompts define an eight-color associative retrieval task. Marked answer positions form the supplied groups and every other active token is a singleton. The locked evaluation contains 80 query specifications, three context seeds, and paired forward and reversed group orders, giving 1,920 prompt instances. The all-singleton control, four repeated exact facts, four distinct related facts, and four relevant answer positions test identity, singleton, and multiple-relevant regimes.

\paragraph{Internal intervention.} At the masked-token query row, we replace the actual softmax computation in all 12 heads of validation-selected layer 5, the final DistilBERT layer. Ordinary attention applies softmax to the token logits. Count correction subtracts the log supplied group size from each member logit. Averaging applies softmax to group-mean logits and combines group-mean projected values. Every other model operation and parameter remains unchanged. Layer and temperature choices were fixed from 24 validation query specifications before the 80 evaluation specifications were generated and hash-locked. Intervals below use 5,000 paired bootstrap resamples of query specifications after averaging context seeds and the two orders.

\begin{table}[t]
\centering
\caption{Frozen DistilBERT layer-6 attention intervention. Means average three
context seeds and two reversed orders within 80 queries; brackets are paired 95\%
bootstrap intervals. $\Delta$ is averaging minus count correction.}
\label{tab:transformer-intervention}
\small
\begin{tabular}{lccc}
\toprule
Condition & $\tau_{\rm cnt}/\tau_{\rm avg}$ & $\Delta$ relevant mass & $\Delta$ candidate prob. \\
\midrule
All singleton & $1/1$ & .00000 & .00000 \\
Repeated facts (4) & $1/1$ & .00032 [.00030,.00035] & -.00215 [-.00253,-.00181] \\
Related (4), selected & $2/1$ & -.00139 [-.00291,.00001] & -.02042 [-.03285,-.00946] \\
Related (4), native & $1/1$ & .00018 [.00017,.00020] & -.00159 [-.00186,-.00133] \\
Four relevant & $1/1$ & -.00083 [-.00093,-.00074] & .00054 [.00014,.00098] \\
\bottomrule
\end{tabular}
\end{table}

\paragraph{Mechanistic result.} At the native temperature, averaging raises singleton relevant attention mass above count correction for every prompt in both repeated and related groups. The mean gains are small but resolved: $3.21\times10^{-4}$ for repeated facts and $1.82\times10^{-4}$ for related facts. Related groups also have larger within-group extremes: the mean maximum-minus-mean logit is $0.443$, compared with $0.391$ for repeated facts, while the corresponding log-mean-exp-minus-mean gaps are $0.0811$ and $0.0676$. Thus the score variation targeted by averaging is present in actual self-attention logits.

Repeated textual facts are not identical score--value copies inside the transformer. Absolute positions and prior contextualization give their marked tokens different logits and projected values. Consequently, the exact-copy equality from the abstract finite-array model applies only to an artificial control with identical logits and values; that control passes to floating-point tolerance, but repetition in a transformer need not satisfy it.

\paragraph{Downstream result and scope.} Increased relevant attention mass does not improve the frozen model's masked-token prediction. At the common native temperature, averaging lowers the correct candidate probability relative to count correction for both repeated and related singleton conditions. After validation-only calibration, related-group count correction selects temperature 2 while averaging selects 1, and averaging is lower by $0.0204$ with paired 95\% interval $[-0.0329,-0.00946]$. Candidate top-1 accuracy is inconclusive, and secondary full-vocabulary probability and rank change direction under separate calibration. Target position also changes the selected-temperature mass comparison. With four relevant positions, averaging lowers total relevant mass, while its candidate-probability increase is only $5.42\times10^{-4}$; ordinary attention has the highest candidate probability in that condition.

The model is barely above the eight-way random baseline: its all-singleton candidate accuracy is $0.131$, versus $0.125$ at chance. Only one small English masked-language model, one synthetic task, the final layer, and oracle-supplied groups are tested, with all heads intervened together. These facts limit what either the downstream nulls or reversals imply about capable retrieval heads. Finite bounded logits on these prompts do not verify the Gaussian phase transition, the all-temperature separation, or the hierarchical extension. The diagnostic establishes only that the two operations are internally distinct and that their downstream effects here are small, metric-dependent, and often adverse.

\paragraph{Evaluation audit.} The locked evaluation used 6,240 forward passes in 165 seconds on a GTX 1080 and produced 11,520 metric rows. The audit checks source and protocol hashes, row uniqueness, probability bounds, temperature locks, singleton equality, native Jensen inequalities, and deterministic reanalysis.

\end{document}